\documentclass{article}

\usepackage{microtype}
\usepackage{graphicx}
\usepackage{subcaption}
\usepackage{booktabs} 

\usepackage{hyperref}
\usepackage{makecell}
\usepackage{comment}
\usepackage{pifont}
\newcommand{\cmark}{\ding{51}} 
\newcommand{\xmark}{\ding{55}}

\usepackage[preprint]{icml2026}

\usepackage{amsmath}
\usepackage{amssymb}
\usepackage{mathtools}
\usepackage{amsthm}

\usepackage{hyperref}   

\usepackage[capitalize,noabbrev]{cleveref}

\theoremstyle{plain}
\newtheorem{theorem}{Theorem}[section]

\newtheorem{lemma}[theorem]{Lemma}

\theoremstyle{definition}
\newtheorem{definition}[theorem]{Definition}
\newtheorem{assumption}[theorem]{Assumption}
\theoremstyle{remark}
\newtheorem{remark}[theorem]{Remark}
\definecolor{rowgray}{gray}{0.95}
\definecolor{crimson}{RGB}{220, 20, 60}
\definecolor{royalblue}{RGB}{65, 105, 225}
\definecolor{darkgrey}{RGB}{169, 169, 169}
\newtheorem{innercustomthm}{Theorem}
\newenvironment{customthm}[1]{\renewcommand\theinnercustomthm{#1}\innercustomthm}{\endinnercustomthm}
\newtheorem{innercustomlem}{Lemma}
\newenvironment{customlem}[1]{\renewcommand\theinnercustomlem{#1}\innercustomlem}{\endinnercustomlem}

\usepackage[textsize=tiny]{todonotes}

\icmltitlerunning{FlexDOME: Near-Constant Strong Violation and Last-Iterate Convergence}

\begin{document}

\twocolumn[
  \icmltitle{Near-Constant Strong Violation and Last-Iterate Convergence\\ for Online CMDPs via Decaying Safety Margins}



  \icmlsetsymbol{equal}{*}

\begin{icmlauthorlist}
    \icmlauthor{Qian Zuo}{inf}
    \icmlauthor{Zhiyong Wang}{math}
    \icmlauthor{Fengxiang He}{inf}
\end{icmlauthorlist}

\icmlaffiliation{inf}{School of Informatics, University of Edinburgh, Edinburgh, United Kingdom}
\icmlaffiliation{math}{School of Mathematics, University of Edinburgh, Edinburgh, United Kingdom}

\icmlcorrespondingauthor{Zhiyong Wang}{zhiyongwangwzy@gmail.com}

\icmlkeywords{Machine Learning, ICML}

\vskip 0.3in
]



\printAffiliationsAndNotice{}  

\begin{abstract}
  We study safe online reinforcement learning in Constrained Markov Decision Processes (CMDPs) under strong regret and violation metrics, which forbid error cancellation over time. Existing primal-dual methods that achieve sublinear strong reward regret inevitably incur growing strong constraint violation or are restricted to average-iterate convergence due to inherent oscillations. To address these limitations, we propose the \textbf{Flex}ible safety \textbf{D}omain \textbf{O}ptimization via \textbf{M}argin-regularized \textbf{E}xploration (FlexDOME) algorithm, the first to provably achieve near-constant $\tilde{O}(1)$ strong constraint violation alongside sublinear strong regret and non-asymptotic last-iterate convergence. FlexDOME incorporates time-varying safety margins and regularization terms into the primal-dual framework. Our theoretical analysis relies on a novel \textit{term-wise asymptotic dominance strategy}, where the safety margin is rigorously scheduled to asymptotically majorize the functional decay rates of the optimization and statistical errors, thereby clamping cumulative violations to a near-constant level. Furthermore, we establish non-asymptotic last-iterate convergence guarantees via a policy-dual Lyapunov argument. Experiments corroborate our theoretical findings.
\end{abstract}

\section{Introduction}

Reinforcement Learning (RL) has achieved remarkable successes in recent years~\citep{liu2024rl,ruiz2025quantum,milani2024explainable}. It formulates sequential decision making as a Markov Decision Process (MDP), where an agent learns a policy to maximize cumulative reward~\citep{sutton1998reinforcement}. However, classical MDPs lack sufficient mechanisms to ensure safety, hindering deployment in safety-critical environments~\citep{garcia2015comprehensive}
. Constrained Markov Decision Processes (CMDPs) address this limitation by incorporating constraints on cumulative costs~\citep{altman1999constrained}.

In online CMDPs, an agent needs to learn in an unknown environment while satisfying safety constraints per episode, which makes safety particularly challenging. Classic reward regret and constraint violation allow cancellations over time~\citep{qiu2020upper,ding2020natural}, obscuring prolonged unsafe behavior, which is unacceptable in safety-critical settings~\citep{fisac2019hierarchical}. This motivates strong metrics, i.e., strong reward regret (sum of positive per-episode suboptimality) and strong constraint violation (sum of positive per-episode violations), with no cancellation \citep{efroni2020exploration}. Such strong safety guarantees naturally arise in settings like power-grid regulation, where cumulative violations induce mechanical or thermal stress, and in clinical control (e.g., automated anesthesia), where even a few severe threshold breaches may trigger irreversible harm~\citep{su2025review,cai2023towards}. 
In these cases, harms cannot be `averaged out,' making strong metrics more suitable than classic ones. In this setting, a fundamental trilemma emerges among (i) stringent safety, (ii) Strong no-regret guarantees, and (iii) last-iterate convergence. Existing approaches are forced to compromise: on the one hand, primal-dual methods achieve last-iterate convergence but their strong violations grow with the number of episodes $T$~\citep{muller2024truly,kitamura2024policy}; on the other hand, methods with tighter regrets often sacrifice last-iterate convergence by applying only to averaged policies~\textcolor{blue}{\citep{stradi2024learning,stradi2025optimal,zhu2025an}}. While achieving stringent safety (e.g., near-constant or even zero violation) is well-studied under the classic regret paradigm~\citep{liu2021learning,bai2022achieving, ma2024learn}, these assurances vanish under the more demanding strong metrics.
This naturally raises a pivotal question: 
\begin{quote}
\centering
    \textit{Can we design a CMDP algorithm that achieves (i) near-constant strong constraint violation, (ii) sublinear strong regret, and (iii) last-iterate convergence?}
\end{quote}
We answer this question affirmatively. To address this challenge, we propose \textbf{Flex}ible safety \textbf{D}omain \textbf{O}ptimization via \textbf{M}argin-regularised \textbf{E}xploration (FlexDOME). FlexDOME advances the standard primal-dual framework by introducing a \textit{dual-dynamic mechanism}: it simultaneously employs a time-decaying safety margin to tighten the feasible set and time-varying regularization terms to stabilize the learning trajectory. Intuitively, the safety margin creates a proactive buffer against uncertainty. 
At early stages of learning, when the uncertainty is high, we take a large margin, steering the agent away from high-risk regions; as information accrues, the margin decays, progressively relaxing conservatism and enabling the pursuit of (possibly) higher-reward policies. The success of this dynamic margin hinges on a stable learning process. Standard primal-dual methods are often plagued by oscillatory dynamics~\citep{efroni2020exploration}. We tame these oscillations by introducing time-varying entropy and $L_2$ regularization to ensure a strongly convex-concave optimization landscape.

We prove that FlexDOME attains $\tilde{O}(1)$ strong constraint violation and $\tilde{O}(T^{5/6})$ strong reward regret, alongside non-asymptotic last-iterate convergence. To the best of our knowledge, this is the first primal-dual algorithm to achieve all three guarantees; see Table \ref{tab:comparison}. Our core theoretical analysis hinges on a novel \textit{term-wise asymptotic dominance strategy}. This is a fundamental departure from the use of safety margins in weak regret analyses, where the total accumulated error can be simply offset by the sum of safety margins~\citep{liu2021learning,kalagarla2025a}. However, we identify that this global compensation strategy fails under strong metrics, where error cancellations are strictly forbidden. Instead, our analysis adopts a \textit{functional perspective}: we treat the per-episode optimization and statistical errors as time-dependent functions and explicitly construct the safety margin to possess a decay rate that is asymptotically slower than or equal to these error functions. By ensuring this functional dominance at the step-wise level, we guarantee that the sequence of positive violations remains summable, thereby clamping the cumulative strong violation to $\tilde{O}(1)$.

Furthermore, the specific decay schedules for the learning rate, regularization terms, and safety margins were \textit{not heuristically pre-determined}. Instead, they emerge as the rigorous analytical solutions to a dynamic optimization problem: simultaneously minimizing the upper bounds for both strong regret and strong violation. Based on Lemma \ref{lem:error bounds}, we identify that these two objectives are conflicting. Thus, the derived decay rates constitute the necessary conditions to resolve this conflict and represent the optimal balance achievable within our framework. Finally, we conduct experiments on tabular CMDPs with both fixed and stochastic thresholds. The results fully corroborate our theoretical claims, demonstrating that FlexDOME maintains near-zero instantaneous violations and achieves near-constant strong violations. We anticipate that our framework can provide new insights for provably safe reinforcement learning.



\paragraph{Related Work.}
Under weak regret metrics, primal-dual methods establish $\tilde{O}(\sqrt{T})$ regret and constraint violation guarantees~\citep{efroni2020exploration}. To enhance safety, subsequent works introduce a safety margin to primal-dual methods, attaining $\tilde{O}(\sqrt{T})$ weak regret and $\tilde{O}(1)$ weak constraint violation guarantees~\citep{liu2021learning,kalagarla2025a}. However, their analysis relies on using the cumulative safety margin to offset the cumulative constraint violation; consequently, the underlying primal-dual dynamics are still prone to the oscillations that preclude guarantees for strong regret or last-iterate convergence.
The allowance for error cancellation makes weak regret an inadequate metric for safety-critical tasks. To address this,~\citet{efroni2020exploration} introduce the more stringent strong regret metric, which accumulates only positive deviations of reward and constraint.~\citet{muller2023cancellation} propose an augmented Lagrangian
method which attains sublinear strong regret/violation with a strictly known safe policy. Relaxing the requirement of a strictly safe policy,~\citet{muller2024truly} and \citet{kitamura2024policy} propose a regularized primal-dual framework to achieve the last-iterate convergence guarantee with strong constraint violation, achieving rates of $\tilde{O}(T^{0.93})$ and $\tilde{O}(T^{6/7})$, respectively.
In parallel,~\citet{stradi2024learning} study adversarial loss with stochastic hard constraints that achieves $\tilde{O}(\sqrt{T})$ weak regret and near-constant strong violation with average-iterate convergence.~\citet{stradi2025optimal} and~\citet{zhu2025an} attain a tighter $\tilde{O}(\sqrt{T})$ strong regret and violation under the guarantee of average-iterate convergence. Table \ref{tab:comparison} summarizes the theoretical results from our work and the most relevant existing methods under strong regret and strong violation metrics.
\begin{table*}[t]
\centering
\setlength{\tabcolsep}{6pt} 
\renewcommand{\arraystretch}{1.1}
\begin{tabular}{lcccccc}
\toprule
\textbf{Algorithm} & \makecell{\textbf{Strong} \\ \textbf{Regret}} & \makecell{\textbf{Strong} \\ \textbf{Violation}} & \makecell{\textbf{Last-iterate} \\ \textbf{Convergence}} &\makecell{\textbf{Unknown} \\ \textbf{Safe Policy}} &\makecell{\textbf{Stochastic} \\ \textbf{Threshold}} \\
\midrule
\hyperlink{cite.muller2023cancellation}{OPTAUG-CMDP} & $\tilde{O}(\sqrt{T})$   & $\tilde{O}(\sqrt{T})$ & \xmark & \xmark & \xmark \\
\hyperlink{cite.stradi2025optimal}{CPD-PO} & $\tilde{O}(\sqrt{T})$   & $\tilde{O}(\sqrt{T})$ & \xmark& \cmark     & \xmark \\
\hyperlink{cite.zhu2025an}{OMDPD} & $\tilde{O}(\sqrt{T})$   & $\tilde{O}(\sqrt{T})$ & \xmark& \cmark     & \xmark \\
\hyperlink{cite.muller2024truly}{RPD-OE} & $\tilde{O}(T^{0.93})$ & $\tilde{O}(T^{0.93})$ & \cmark& \cmark & \xmark \\
\hyperlink{cite.kitamura2024policy}{UOpt-RPGP} & $\tilde{O}(T^{6/7})$  & $\tilde{O}(T^{6/7})$ & \cmark& \cmark & \xmark\\ 
\midrule
\textbf{FlexDOME} &  $\tilde{O}(T^{5/6})$ &  $\tilde{O}(1)$ &  \cmark& \cmark &  \cmark \\
\bottomrule
\end{tabular}
\caption{Comparison between FlexDOME and related work under strong regret and violation metrics. For clarity, dependencies on the state space ($S$), action space ($A$), and horizon ($H$) are omitted here.}
\label{tab:comparison}
\end{table*}
\section{Preliminaries}\label{sec:preliminary}
\paragraph{Notation.} For any $x \in \mathbb{R}$, we define the operation $[x]_{+}:=\max\{0,x\}$ to be the positive truncation of $x$. We use $O(\cdot)$ and $\Omega(\cdot)$ to denote asymptotic upper and lower bounds, respectively, and $\Theta(\cdot)$ when a bound is asymptotically tight. The symbol $\tilde O(\cdot)$ hides polylogarithmic factors, and $\lesssim$ denotes inequality up to constants and polylogarithmic factors.
\paragraph{Constrained Markov decision process (CMDP).}
We consider a finite-horizon Markov decision process (MDP), where the state space is denoted by $\mathcal{S}$ (with finite cardinality $S$), the action space by $\mathcal{A}$ (with finite cardinality $A$), and the horizon by $H$.  
At step $h\in[H]$, the agent occupies state $s_h\!\in\!\mathcal{S}$, takes action $a_h\!\in\!\mathcal{A}$, and the subsequent state $s_{h+1}$ is sampled from the transition probability
$p:\mathcal{S}\times\mathcal{A}\times\mathcal{S}\to[0,1]$.  
$r_h:\mathcal{S}\times\mathcal{A}\to[0,1]$ represents the reward function at each step $h$. A policy $\pi=(\pi_1,\dots,\pi_H)\in\Pi$ specifies a distribution $\pi_h(\cdot\mid s)\in\Delta(\mathcal{A})$ for every state–step pair, where $\Pi:=\{(\pi_1,\cdots, \pi_H)\mid\forall h,s:\pi_h(\cdot\mid s)\in\Delta(\mathcal{A})\}$. A Constrained MDP augments this setting with $m$ constraints.  
For constraint $i\in[m]$ and step $h$, a constraint $d_{i,h}(s,a)\in[0,1]$ is incurred; the cumulative expectation must not fall below a given threshold $\alpha_i\in[0,H]$. Thus, a CMDP can be fully characterized by $\mathcal{M}=(\mathcal{S},\mathcal{A},H,p,r,d,\alpha)$.

In this work, we adopt a generalized online learning setting where the safety thresholds are stochastic rather than fixed constants. In this setting, the agent must estimate the expected value of rewards, constraints, and thresholds simultaneously from interactions. Specifically, at each interaction $(s,a)$ for step $h$ and episode $t$, the agent observes stochastic samples: a reward $\tilde{r}_h^t(s, a)$, constraints $\{\tilde{d}_{i,h}^t(s, a)\}_{i=1}^m$, and thresholds $\{\tilde{\alpha}_{i,h}^t\}_{i=1}^m$ which are state-action independent. These are drawn from stationary but hidden distributions $\mathcal{R}$, $\{\mathcal{G}_i\}_{i=1}^m$ and $\{\mathcal{L}_i\}_{i=1}^m$, respectively. The formal definition of this setting is detailed in Appendix~\ref{app:distionction}.

\begin{remark}
    This formulation generalizes the standard CMDP setting. The fixed-threshold scenario can be recovered as a special instance where the threshold distribution is a Dirac delta function centered at a constant. Consequently, our theoretical analyzes and results apply directly to the standard fixed-threshold setting without loss of generality.
\end{remark}

\paragraph{Value and objective functions.}
For any vector $v\in[0,1]^{\mathcal{S}\times\mathcal{A}}$ and policy $\pi\in \Pi$, consider the value functions  
\begin{align*}
 V_{v,h}^{\pi}(s)&=\mathbb{E}_{\pi,p}\!\bigl[\sum_{h'=h}^{H} v_{h'}(s_{h'},a_{h'}) \mid s_h=s\bigr],\\
  Q_{v,h}^{\pi}(s,a)&=\mathbb{E}_{\pi,p}\!\bigl[\sum_{h'=h}^{H} v_{h'}(s_{h'},a_{h'}) \mid s_h=s,a_h=a\bigr],   
\end{align*}
where $V_{v,h}^{\pi}(s)$ denotes the expected sum of $v$ from step $h$ onward given $s_h = s$, and $Q_{v,h}^{\pi}(s,a)$ denotes the same expectation further conditioned on $a_h = a$. For notational brevity,
set $V_v^\pi:=V_{v,1}^\pi(s_1)$.  
The objective is to find a policy solution $\pi^\star$ to the following policy optimization problem,
\begin{equation}
\max_{\pi \in \Pi}\;V_r^\pi
\quad\text{subject to}\quad
V_{d_i}^\pi\;\ge\; \alpha_i\quad (\forall i\in [m]),
\quad
\label{eq:CMDP-problem}
\end{equation}
which identifies a policy that maximizes the expected cumulative reward while ensuring that the expected cumulative value of each constraint signal satisfies its threshold.

\paragraph{Training protocol.}
Across $T$ episodes, a policy $\pi_t$ is selected at the beginning of episode $t$ and executed for $H$ steps. The goal is to simultaneously minimize its strong reward regret and strong constraint violation,
\begin{align*}
    \mathcal{R}_T(r) &:= \sum_{t=1}^{T} \Big[ V_r^{\pi^\star} - V_r^{\pi_t}\Big]_+,\\\text{      }
    \mathcal{R}_T(d) &:= \max_{i \in [m]} \sum_{t=1}^{T} \Big[ \alpha_i-V_{d_i}^{\pi_t}\Big]_+.
\end{align*}
These expressions measure the cumulative sum of only the positive deviations, capturing how much the reward underperforms the optimal or how much the constraints are violated in each episode. 
Each positive error contributes its full amount to the total, and no future episode can offset it. Throughout, we assume the following Slater condition, which is mild as it holds when there exists some (unknown) strictly feasible policy~\citep{efroni2020exploration, qiu2020upper,ying2022dual,ding2023last,kitamura2024policy}.

\begin{assumption}\label{assumption} 
There exists an unknown policy $\pi^0 \in \Pi$ such that $V_{d_i}^{\pi^0}\ge d_i^0$, where $d_i^0>\alpha_i$ for all $i\in[m]$. Set the Slater gap $\Xi:=\min_{i\in[m]}\{d_i^0-\alpha_i\}$.
\end{assumption}

\section{FlexDOME}\label{sec:primal-dual scheme}

This section introduces our algorithm, and the designs behind to ensure the near-constant strong violation, sublinear strong regret and last-iterate convergence.

\subsection{The Primal-Dual Scheme in FlexDOME}
To address the fundamental trilemma, FlexDOME employs a dual-dynamic mechanism: it deploys a decaying safety margin to proactively buffer against uncertainty, in tandem with time-varying regularization to induce the geometric stability essential for last-iterate convergence.

\paragraph{Decaying Safety margin.} Our core idea is to proactively establish a `margin of safety' to mitigate the effects of uncertainty in guaranteeing safety. We translate this idea into a formal mechanism by first introducing a time-decaying safety margin $\epsilon_{i,t}$ for each episode $t$ and constraint $i$ into the original optimization problem (\ref{eq:CMDP-problem}):
\begin{equation}\label{eq:CMDP-2}
\max_{\pi\in \Pi}\;V_{r}^\pi
\quad
\text{s.t.}
\quad
V_{d_i}^\pi\ge
\alpha_{i}+\epsilon_{i,t} \quad(\forall i\in [m]),
\end{equation}
where the constraints are tightened by the safety margins 
to enhance safety during learning. The corresponding Lagrangian function is defined as follows:
\begin{equation*}
    \mathcal{L}_t(\pi,\lambda):=
V_{r}^\pi+\sum_{i=1}^m
\lambda_i\left(V_{d_i}^\pi-\epsilon_{i,t}-\alpha_{i}\right),
\end{equation*}
where $\lambda=[\lambda_1,\dots,\lambda_m]^\top\in\mathbb{R}_+^m$ is the vector of non-negative dual variables (or Lagrange multipliers), with each $\lambda_i$ corresponding to the $i$-th constraint.

\paragraph{Time-Varying Regularizations.} While the safety margin creates a dynamic buffer, standard primal-dual CMDP formulations lack strong convexity-concavity, which can cause oscillatory dynamics~\citep{stooke2020responsive}. These oscillations can breach a simple safety buffer, and thus fail to achieve stringent safety guarantees~\citep{moskovitz2023reload,muller2024truly}.
To overcome this limitation, we introduce a time-varying regularization framework that provides the geometric stability necessary for the safety margin to be effective. By augmenting the Lagrangian with dynamically scaled entropy and $\ell_2$-norm penalties, we reshape the optimization landscape. Entropy regularization, $\mathcal{H}(\pi)$, ensures the primal objective is strongly concave, preventing extreme policy updates. The $\ell_2$ penalty, guarantees the dual objective is strongly convex, reducing gradient oscillations. Together, these components create a strongly convex-concave structure. The resulting regularized Lagrangian for regularization parameter $\tau_t>0$ at episode $t$ is formulated as:
\begin{align}
        \mathcal{L}_{\tau_{t},t}(\pi, \lambda) 
        :=&V_{r}^\pi+
\lambda^\top\Bigl(V_{d}^\pi-\epsilon_t-\alpha\Bigr)\nonumber\\&+ \tau_{t}\left(\mathcal{H}(\pi) + \frac{1}{2} \|\lambda\|^2\right),
\end{align}
where $\mathcal{H}(\pi) := -\mathbb{E}_{\pi}\left[\sum_{h=1}^{H} \log(\pi_h(a_h|s_h)) \right]$ is the policy entropy and $\epsilon_t$ denotes the vector of safety margins. The objective is to find the saddle point of this regularized problem over the policy space $\Pi$ and a compact dual domain $\mathcal{C}:=[0,4H/\Xi]^m$:
\begin{equation}\label{eq:opt-ques}
    \max_{\pi \in \Pi} \min_{\lambda \in \mathcal{C}} \mathcal{L}_{\tau_{t},t}(\pi, \lambda).
\end{equation}
The strongly convex-concave structure guarantees that this problem has a unique saddle point, $\left(\pi_{\tau_{t},\epsilon}^{\star},\lambda_{\tau_{t},\epsilon}^\star\right)$, which we define as the regularized optimizer for episode $t$.

\subsection{Estimates}\label{sec:estimate}
FlexDOME employs a hybrid estimation strategy to navigate the unknown environment. It constructs optimistic estimates for rewards, constraints, and the entropy term to encourage exploration, while the transition model and thresholds are unbiasedly estimated directly from empirical data. Let $(s_h^l, a_h^l)$ denote the state-action pair visited in episode $l$ at step $h$. The term $\mathbf{1}_{\{\cdot\}}$ is the indicator function; thus, $N_h^{t-1}(s,a)=\sum_{l=1}^{t-1}\mathbf{1}_{\{s_h^l=s,a_h^l=a\}}$ is the total number of visits to $(s,a)$ at step $h$ before episode $t$. Then, the empirical averages for rewards, constraints, thresholds and transition probabilities can be calculated as follows: 
\begin{equation}\label{eqs:empirical}
\begin{aligned}
    \hat{r}_h^{t-1}(s,a) &:= \frac{\sum_{l=1}^{t-1} \tilde{r}_{h}^l(s,a)\,\mathbf{1}_{\{s_h^l=s,a_h^l=a\}}}{\max\{1, N_h^{t-1}(s,a)\}},\\
    \quad \hat{d}_{i,h}^{t-1}(s,a) &:= \frac{\sum_{l=1}^{t-1} \tilde{d}_{i,h}^l(s,a)\,\mathbf{1}_{\{s_h^l=s,a_h^l=a\}}}{\max\{1, N_h^{t-1}(s,a)\}}, \\
    \hat{\alpha}_{i}^{t-1} &:= \frac{\sum_{l=1}^{t-1}\sum_{h=1}^H \tilde{\alpha}_{i,h}^l}{(t-1)H}, \\
    \quad \hat{p}_h^{t-1}(s'\mid s,a) &:= \frac{\sum_{l=1}^{t-1} \mathbf{1}_{\{s_h^l=s,a_h^l=a,s_{h+1}^l=s'\}}}{\max\{1, N_h^{t-1}(s,a)\}}.
\end{aligned}
\end{equation}
We then construct the estimators for use in episode $t$. The safety threshold is estimated as the global empirical average of all historical observations: $\overline{\alpha}_{i}^t:= \hat{\alpha}_i^{t-1}$. As each true threshold is constant, this method is data-efficient and yields an estimate that is independent of any specific state-action pair. The remaining state-action dependent estimators are constructed as follows:
\begin{equation}\label{eqs:opt}
\begin{aligned}
    \overline{r}_h^t(s,a) &:= \hat{r}_h^{t-1}(s,a)+\phi_h^{t-1}(s,a),\\
    \overline{d}_{i,h}^t(s,a)&:= \hat{d}_{i,h}^{t-1}(s,a)+\phi_h^{t-1}(s,a),\\
    \overline{\psi}_{h}^t(s,a)&:= -\log(\pi_h^t(a\!\mid\!s))+\phi_{h}^{p,t-1}(s,a)\log(A),\\\overline{p}_h^t(s'|s,a)&:=\hat{p}_h^{t-1}(s'|s,a).
\end{aligned}
\end{equation}
The bonus term $\phi_h^{t}(s,a) := \phi_{h}^{r,t}(s,a) + \phi_{h}^{p,t}(s,a)$ combines the uncertainties from both rewards and transition estimations, where for any confidence parameter $\delta\in (0,1)$, the reward bonus is $\phi_h^{r,t}(s,a)=O\left(\sqrt{\frac{\log(mSAHT/\delta)}{\max\{1, N_h^t(s,a)\}}}\right)$ and the transition bonus is $\phi_h^{p,t}(s,a)=O\left(H\sqrt{\frac{S+\log(SAHT/\delta)}{\max\{1, N_h^t(s,a)\}}}\right)$.

\subsection{Learning Algorithm} 

We now present \textbf{FlexDOME}, 
detailed in Algorithm \ref{alg:main}. In each episode $t$, the algorithm first constructs an optimistic empirical CMDP $\mathcal{M}_t:=(\mathcal{S},\mathcal{A},H,\overline{p}_t,\overline{r}_t,\overline{d}_t,\overline{\alpha}_t)$, using the estimators from Section \ref{sec:estimate}. It then performs policy evaluation. To prevent optimistic bonuses from inflating value estimates unboundedly, we use a Truncated Policy Evaluation (TPE) routine~\citep{efroni2020exploration}. TPE computes $V$-values for the constraint estimates and $Q$-values for the composite objective $\overline{y}_t := \overline{r}_t + \lambda_t^\top \overline{d}_t + \tau_t \overline{\psi}_t$, which aggregates the optimistic estimates of the reward, constraints, and entropy. 
See Algorithm~\ref{alg:TPE} in Appendix~\ref{app:property} for details. Based on this, FlexDOME executes a single primal-dual update: the policy (primal variable) is updated via mirror ascent, and the dual variables are updated via projected gradient descent. The resulting policy is then deployed to collect new data for the next iteration.

\begin{algorithm}[t]
\caption{FlexDOME}\label{alg:main}
\begin{algorithmic}[1] 
    
    \STATE \textbf{Input:} $\mathcal{C}=[0,\frac{4H}{\Xi}]^{m}$, stepsize $\eta_{t}$, regularization $\tau_{t}$, number of episodes $T$, safety margin $\epsilon_{i,t}\;(\forall i)$
    
    \STATE \textbf{Initialize:} policy $\pi_{1,h}(a\mid s)=\frac1A\;(\forall s,a,h)$, $\lambda_{1}=\textbf{0}\in\mathbb{R}^{m}$
    
    \FOR{$t = 1$ \textbf{to} $T$}
        \STATE Update estimators $\overline{r}_{t}$, $\overline{d}_{t}$, $\overline{\alpha}_t$, $\overline{\psi}_{t}$, and $ \overline{p}_{t}$
        
        \STATE Truncated policy evaluation (Algorithm~\ref{alg:TPE}) for $\overline{y}_{t}$ and $\overline{d}_{t}$: $\left(\hat{Q}^{t}_{\bar{y}_{t}}(\cdot),\hat{V}^{t}_{\bar{d}_{t}}\right)\gets\textup{TPE}\!\bigl(\pi_{t},\lambda_{t}, \overline{r}_{t},\overline{d}_{t},\overline{\psi}_{t},\overline{p}_{t}\bigr)$
        
        \STATE Policy Update ($\forall h,s,a$): $\displaystyle
            \pi_{t+1,h}(a\!\mid\! s)\propto \pi_{t,h}(a\!\mid\!s)
            \exp\!\bigl(\eta_{t}\,\hat{Q}^{t}_{h,\bar{y}_{t}}(s,a)\bigr)
            $
            
        \STATE Dual Update: $\displaystyle
            \lambda_{t+1}\gets \text{Proj}_{\mathcal{C}}
            \!\Bigl((1-\eta_{t}\tau_{t})\lambda_{t}
            -\eta_{t}\bigl(\hat{V}^{t}_{\bar{d}_{t}}-\epsilon_t-\overline{\alpha}_t\bigr)\Bigr)$
            
        \STATE Rollout $\pi_{t}$ and update counters and empirical model (i.e., $\hat{r}_{t},\hat{d}_{t},\hat{\alpha}_t, \hat{p}_{t}, N_t$)
    \ENDFOR
\end{algorithmic}
\end{algorithm}

\section{Theoretical Analysis}

This section establishes the theoretical guarantees for Algorithm \ref{alg:main}. We first present our main results on strong regret and violation bounds, followed by our practical guarantee of last-iterate convergence. We then detail the key technical lemmas that underpin these results. The full proofs for this section are deferred to Appendix \ref{sec:regrets} and Appendix \ref{app:last-iterate}.
 
\subsection{Strong Regret Bounds}\label{sec:sgbound}

We first provide the main theoretical results for FlexDOME.
\begin{theorem}[Strong regret bounds for reward and violation]\label{thm: bounded-cv}
    For any confidence parameter $\delta\in(0,1)$, let $\eta_{t}=t^{-5/6}$, $\tau_{t}=t^{-1/6}$, and $\epsilon_{i,t} =18/5\sqrt{H^3C_B}\left(t^{-1/6}\cdot \log(4SAHt/\delta)^{1/4}\right)$ for any constraint $i$. Then, with probability at least $1-\delta$, Algorithm \ref{alg:main} achieves the following bounds:
    \begin{equation*}
        \mathcal{R}_{T}(r)\le\tilde{O}(T^{5/6}) \quad and \quad \mathcal{R}_{T}(d)=\tilde{O}(1),
    \end{equation*}
    where $T$ denotes the number of episodes, $C_B=O(m,S,A,H)$ is a $T$-independent constant and $\tilde{O}$ hides polylogarithmic factors in $(S,A,H,m,\log(T),\log(\frac{1}{\delta}),\Xi)$.
\end{theorem}
Theorem \ref{thm: bounded-cv} establishes near-constant strong constraint violation and sublinear $\tilde{O}(T^{5/6})$ strong regret. 
We highlight that the attainment of near-constant strong violation here is not merely a consequence of the safety margin, but relies on the intricate synergy between the learning rate $\eta_t$, regularization coefficient $\tau_t$, and the safety margin $\epsilon_{i,t}$. Crucially, the specific decay schedules for these parameters were \textit{not heuristically pre-determined}. Instead, they were derived as the rigorous analytical solution to the trade-off between minimizing strong regret and suppressing cumulative violations. By modeling $\eta_t$, $\tau_t$, and $\epsilon_{i,t}$ as time-dependent functions within our analysis, we identified that the exponents $-5/6$, $-1/6$, and $-1/6$ respectively constitute the necessary conditions to neutralize per-episode optimization and statistical errors. These specific rates represent the optimal balance achievable within this regularized primal-dual framework; simply adjusting the decay rate of these parameters is insufficient to reach $\tilde{O}(\sqrt{T})$ strong regret.

\begin{remark}
While the $\tilde{O}(T^{5/6})$ regret leaves a gap to the optimal $\tilde{O}(\sqrt{T})$, it represents a critical trade-off for safety. To our knowledge, this is the first result to achieve near-constant strong violation ($\tilde{O}(1)$) within the last-iterate convergence regime. In contrast, prior last-iterate methods inherently incur polynomial violation growth, such as $\tilde{O}(T^{0.93})$ in \citet{muller2024truly} and $\tilde{O}(T^{6/7})$ in \citet{kitamura2024policy}.
\end{remark}



\subsection{Last-iterate convergence}\label{sec:last-iterate}

Beyond the regret and violation bounds, we prove that FlexDOME achieves a more stringent last-iterate convergence where the per-step violation becomes exactly zero rather than asymptotically vanishing, which is crucial for  practical deployment~\citep{ding2023last}. We present the following theorem.

\begin{theorem}[Last-iterate convergence]\label{thm:last-iterate}
Conditioned on Assumption \ref{assumption}, for small $\varepsilon>0$ and $t=\Omega(\varepsilon^{-4}\log(1/\varepsilon))$, if $\eta_{t}=\Theta(\varepsilon^{3})$, $\tau_{t}=\Theta(\varepsilon)$ and $\epsilon_{i,t}=\Theta(\varepsilon)$ for all constraint $i$, then we have
\begin{equation*}
\bigl[V_{r}^{\pi^\star}-V_{r}^{\pi_t}\bigr]_{+}\le\Theta(\varepsilon),
\quad
\bigl[\alpha_{i}-V_{d_i}^{\pi_t}\bigr]_{+}=0\quad
(\forall\,i\in[m]).    
\end{equation*}
\end{theorem}
Theorem \ref{thm:last-iterate} demonstrates that the final policy is guaranteed to be both $\varepsilon$-optimal and strictly constraint-satisfying. Departing from standard bounds that merely limit violation to $\Theta(\varepsilon)$, we prove that our algorithm eliminates violation entirely after $\Omega(\varepsilon^{-4}\log(1/\varepsilon))$ iterations. This strict zero-violation property is the linchpin of our near-constant cumulative violation guarantee.

\begin{remark}
     We emphasize the important role of last-iterate convergence within our theoretical framework. While prior works have achieved $\tilde{O}(\sqrt{T})$ strong regret under the framework of average convergence, this may obscure potential safety risks in the final policy. The ability to rigorously prove last-iterate convergence marks a fundamental distinction in theoretical properties; it is the key to ensuring that the final policy possesses both optimality and strict safety for practical deployment. 
\end{remark}


\subsection{Analysis Sketch}\label{sec:key lemmas} 
This section outlines the core technical arguments underpinning our main theorem. First, we leverage the convergence properties of a policy-dual potential function inspired by~\citet{ding2023provably} and~\citet{muller2024truly}, which serves as a Lyapunov measure to track the learning dynamics. Second, and crucially, we rigorously characterize the per-episode safety-performance trade-off by constructing a mixed policy and a dual perturbation vector. These constructions enable us to explicitly decompose the reward and violation errors and link them directly to the potential function and the safety margin.

We begin by introducing the policy-dual divergence potential function as follows:
\begin{align*}
  \Phi_{t}
\;=&\;
\sum_{s,h}\mathbb{P}_{\pi_{\tau_{t},\epsilon}^{\star}}[s_h = s]\,
\mathrm{KL}\Bigl(\pi_{\tau_{t},\epsilon,h}^\star(\cdot\!\mid\!s),\pi_{t,h}(\cdot\!\mid\!s)\Bigr)
\\&+
\frac12\bigl\|\lambda_{\tau_{t},\epsilon}^\star - \lambda_{t}\bigr\|^2. 
\end{align*}
It quantifies how closely the current policy-dual iterate $(\pi_t, \lambda_t)$ approximates the optimal margin-regularized policy-dual pair $(\pi_{\tau_{t},\epsilon}^{\star}, \lambda_{\tau_{t},\epsilon}^\star)$.
We prove that this function contracts at each step.

\begin{lemma}[Convergence]\label{lem:convergence}
Let $\eta_{t}$, $\tau_{t}\le1$ and a confidence parameter $\delta\in(0,1)$. With probability at least $1-\delta$, the policy-dual divergence potential of Algorithm \ref{alg:main} holds
\begin{align*}
   &\Phi_{t+1}\le \exp\left(-\sum_{j=1}^t\eta_j\tau_j \right)\Phi_1+\frac{HC+D}{2}\sum_{j=1}^t\eta_j^2\\&\cdot\exp\left(-\sum_{k=j+1}^t\eta_k\tau_k\right)+\sum_{j=1}^t\eta_j\delta_j\exp\left(-\sum_{k=j+1}^t\eta_k\tau_k\right),
\end{align*}
where $C=\exp{\left(\eta_{t} H\left(1+\frac{4mH}{\Xi}+\tau_{t}\log(A)\right)\right)}(2A^{\eta_{t}\tau_{t}}H^2$\\$
\left(1+\frac{4mH}{\Xi}+\tau_{t}\log(A)\right)^2 + \frac{128\tau_{t}^2\sqrt{A}}{e^2})$, $D=m\bigl(H + \tau_{t}\,\left(\frac{4H}{\Xi}\right)\bigr)^2$ and $\delta_j=\hat{V}_{\overline{y}_j}^j-V_{y_j}^{\pi_j}+\sum_i\frac{4H}{\Xi}\left(\hat{V}_{\bar{d}_{i,j}}^j-V_{d_i}^{\pi_j}\right)$.
\end{lemma}
Lemma \ref{lem:convergence} shows that the iterates of FlexDOME contract towards a neighborhood of the margin-regularized saddle point. The upper bound consists of three primary components: (i) a decaying term dependent on the initial potential $\Phi_1$; (ii) the accumulated optimization error from the primal-dual updates; and (iii) the statistical error from estimating the unknown CMDP model. 

\begin{remark}
    Our convergence analysis extends the framework of~\citet{muller2024truly} by accommodating time-varying learning rates $\eta_t$ and regularization terms $\tau_t$, as opposed fixed constants. This adaptation is critical; without it, merely incorporating a dynamic safety margin fails to achieve near-constant violation and improved regret bounds.
\end{remark}

To bridge the gap to the original CMDP, our analysis divides the episodes into two parts, divided by $C''=O\bigr((H^3C_B)^3\log^{3/2}(H^3C_B)\bigl)$. For episodes $t< C''$, the margin may be large, so we bound the per-episode regret by $H$. For episodes $t\ge C''$, the decaying margin is guaranteed to be sufficiently small such that $\epsilon_{i,t}\le \Xi/2$. Consequently, for this regime, the optimization problem (\ref{eq:opt-ques}) has at least one feasible solution by Assumption \ref{assumption} and exhibits strong duality. Our main technical lemmas are therefore derived for these episodes. We now introduce a key lemma that links per-episode performance guarantees to the learning dynamics and the safety buffer.

\begin{lemma}[Per-episode trade-off]\label{lem:error bounds}
For any $t\ge C''$, any constraint $i$ and any sequence $\{\pi_t\}_{t\in[T]}$, it holds
\begin{align*}
    &\bigl[V_{r}^{\pi^\star}-V_{r}^{\pi_t}\bigr]_{+}\le
H^{3/2}\,\bigl(2\Phi_t\bigr)^{1/2}+ H\log(A)\tau_{t}+\frac{H}{\Xi}\epsilon_{i,t},\\
&\max_{i\in[m]}\left[\alpha_{i}-V_{d_i}^{\pi_t}\right]_{+} \le \left[H^{3/2}\,\bigl(2\Phi_t\bigr)^{1/2} +\frac{4H}{\Xi}\tau_{t}-\epsilon_{i,t}\right]_{+}.
\end{align*}
\end{lemma}
This lemma is crucial for our main theorem and establishes that the per-episode performance gap and constraint violation can be explicitly bounded by $\Phi_t$, $\tau_t$, and $\epsilon_{i,t}$.
To see this, we recall the reward decomposition
\begin{equation*}
V_{r}^{\pi^\star}-V_{r}^{\pi_t} = \underbrace{\left(V_{r}^{\pi^\star}-V_{r}^{\pi_{\tau_t,\epsilon}^{\star}}\right)}_{\text{Approximation Error}} + \underbrace{\left(V_{r}^{\pi_{\tau_t,\epsilon}^{\star}}-V_{r}^{\pi_t}\right)}_{\text{Optimization Error}}
\end{equation*}
The first term captures the approximation error induced by the safety buffer. By constructing the probabilistic mixed policy $\pi^{\text{mix}}=(1-\frac{\epsilon_{i,t}}{\Xi})\pi^\star+\frac{\epsilon_{i,t}}{\Xi}\pi^0,$ we enforce strict feasibility in the tightened domain. By leveraging the feasibility of $\pi^{\text{mix}}$ in the tightened domain and the optimality of $\pi_{\tau_t,\epsilon}^\star$ for the regularized Lagrangian, we obtain $V_r^{\pi^{\star}}-V_r^{\pi_{\tau_{t},\epsilon}^\star} \le \frac{\epsilon_{i,t}}{\Xi}(V_r^{\pi^\star}-V_r^{\pi^0}) + \tau_t H \log A$.
The second term, $V_{r}^{\pi_{\tau_t,\epsilon}^{\star}} - V_{r}^{\pi_t}$, represents the optimization gap. By invoking the difference lemma and Pinsker’s inequality, we obtain $V_{r}^{\pi_{\tau_t,\epsilon}^{\star}} - V_{r}^{\pi_t} \le H^{3/2}\sqrt{2\Phi_t}$.

Conversely, the violation relies on the decomposition
\begin{equation*}
    \alpha_i - V_{d_i}^{\pi_t} =\underbrace{(\alpha_i - V_{d_i}^{\pi_{\tau_t,\epsilon}^{\star}})}_{\text{Safety Margin Buffer}} + \underbrace{(V_{d_i}^{\pi_{\tau_t,\epsilon}^{\star}} - V_{d_i}^{\pi_t})}_{\text{Optimization Error}}.
\end{equation*}
For the violation analysis, the crux lies in bounding the term $\alpha_i - V_{d_i}^{\pi_{\tau_t,\epsilon}^{\star}}$.
For any dual variable $\lambda$, we have $(\lambda-\lambda_{\tau_t,\epsilon}^{\star})^\top (V_{d}^{\pi_{\tau_t,\epsilon}^{\star}} -\alpha -\epsilon_t ) + \frac{\tau_{t}}{2} (\|\lambda\|^2 - \|\lambda_{\tau_t,\epsilon}^{\star}\|^2)\ge 0$. Constructing a specific $\lambda$ vector by choosing $\lambda_j=\lambda_{\tau_{t},\epsilon,j}^{\star}$ for all $j\neq i$, and $\lambda_i>\lambda_{\tau_t,\epsilon,i}^\star$ for any constraint $i$, and rearranging the inequality, we obtain $\alpha_i-V_{d_i}^{\pi_{\tau_{t},\epsilon}^{\star}} \le \frac{\tau_t}{2}\left(\lambda_i + \lambda_{\tau_t,\epsilon, i}^{\star}\right) - \epsilon_{i,t}.$

Crucially, the safety margin $\epsilon_{i,t}$ appears with a negative sign. The second term is analogous to the reward case. Combining these terms, we derive the final trade-off: a larger safety margin increases the reward regret via the approximation error but suppresses the violation via the buffer effect.

\subsection{Proof of the Main Theorem}\label{sec:proof}
We first present the proof for the strong constraint violation, which represents the primary technical novelty of our regret analysis.
\paragraph{Decomposition and Dominance Strategy.} A distinct feature of our proof, which fundamentally departs from standard weak-regret analyses, is the shift from global error cancellation to \textit{term-wise asymptotic dominance}. Prior works ~\citep{liu2021learning,kalagarla2025a} typically rely on the total accumulated safety margin $\sum_t \epsilon_{t}$ to absorb the total accumulated error, which is insufficient for strong regret. Instead, our analysis necessitates a term-wise control strategy that scrutinizes the asymptotic decay rate (w.r.t. $t$) of each error component. To achieve this, we construct the safety margin as a composite barrier, $\epsilon_{i,t} := \sum_{k=1}^4 \epsilon_{i,t}^{(k)}$, where each component is calibrated to decay asymptotically slower than or equal to its corresponding error source. This ensures the margin asymptotically `envelopes' the uncertainties, strictly clamping the violation to a near-constant level. Specifically, by Lemmas \ref{lem:error bounds}, the cumulative violation can be upper bounded by
\begin{equation}\label{eq:violation_decomposition} \mathcal{R}_T(d)\le (C''-1)H+\sum_{t=C''}^T\left[H^{3/2}\sqrt{2\Phi_t}+\frac{4H}{\Xi}\tau_{t}-\epsilon_{i,t}\right]_+. \end{equation}
Substituting the explicit expansion of $\Phi_t$ (Lemma \ref{lem:convergence}) into Equation \eqref{eq:violation_decomposition}, we analyze the following four terms:
\begin{align*}
& \left[c_1\exp\left(-\sum_{j=1}^t\frac{\eta_{j}\tau_{j}}{2}\right)-\epsilon_{i,t}^{(1)}\right]_+, \tag{a}\\
& \left[c_2\left(\sum_{j=1}^t \eta_j^2\exp\left(-\sum_{k=j+1}^t \eta_k\tau_k\right)\right)^{1/2}-\epsilon_{i,t}^{(2)}\right]_+,\tag{b} \\
& \left[c_3\left(\sum_{j=1}^t \eta_j\delta_j\exp{\left(-\sum_{k=j+1}^t\eta_k\tau_k\right)}\right)^{1/2}-\epsilon_{i,t}^{(3)}\right]_+,\tag{c} \\
& \left[c_4\tau_{t}-\epsilon_{i,t}^{(4)}\right]_+.\tag{d}
\end{align*}
Here, $c_i$ is a problem-dependent constant defined in Appendix \ref{app:property} for $i\in[4]$.  

\begin{figure*}[h]
    \centering
    \includegraphics[width=0.8\linewidth]{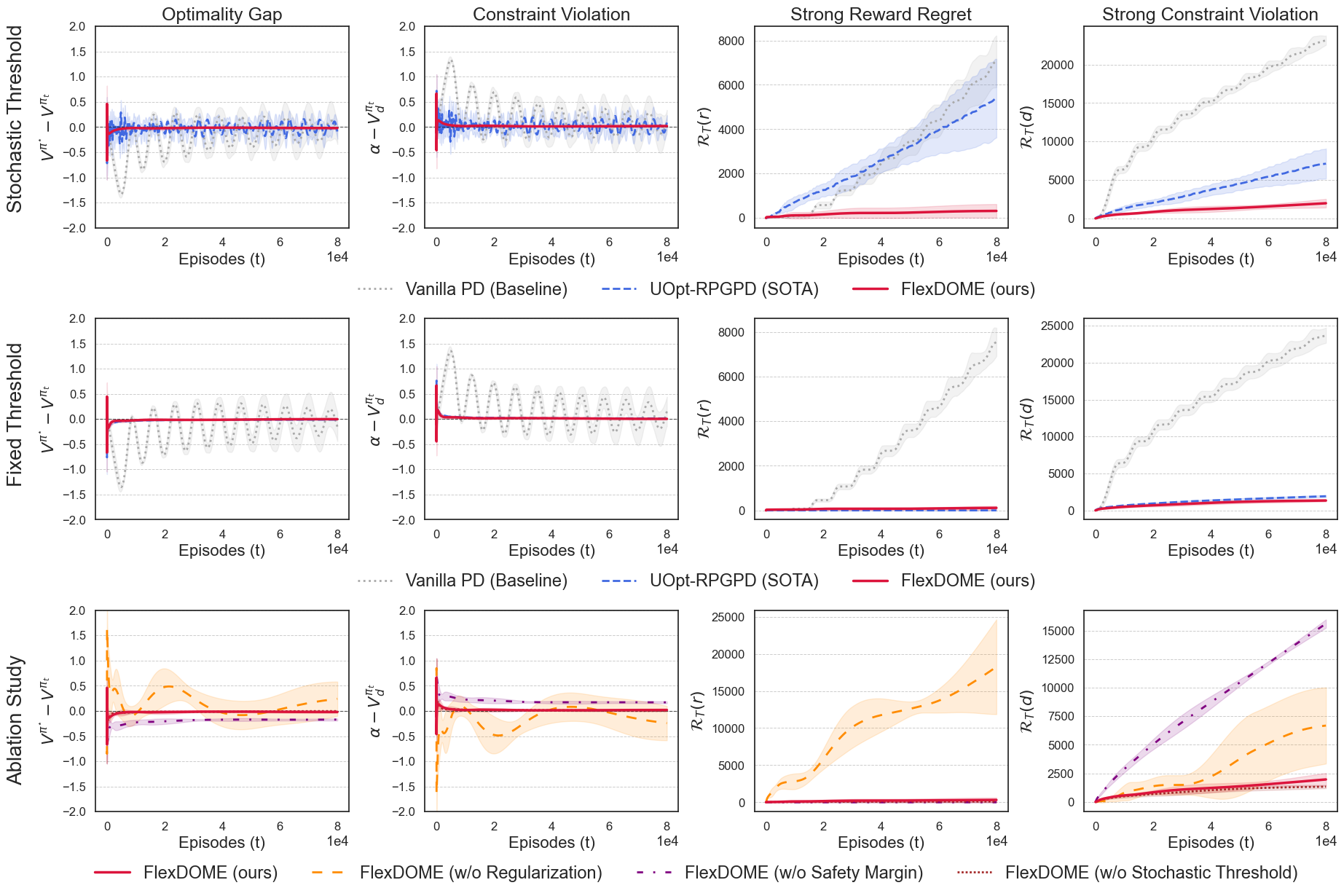}
    \captionsetup{skip=2pt}
    \caption{Performance comparison of \textbf{\textcolor{crimson}{FlexDOME}} (ours) against \textbf{\textcolor{royalblue}{UOpt-RPGPD}} and \textbf{\textcolor{darkgrey}{Vanilla PD}} baselines under both stochastic-threshold (\textbf{top row}) and fixed-threshold (\textbf{middle row}) settings. The \textbf{bottom row} presents an ablation study on key components of our method: the safety margin, regularization, and the stochastic threshold mechanism. All plots show the mean and standard error over 5 seeds. Performance is measured by the instantaneous optimality gap and constraint violation, alongside their corresponding strong regrets.}
    \label{fig:1}
\end{figure*}
\paragraph{Asymptotic Rate Analysis.} We now establish the asymptotic convergence rates for these error terms to determine the necessary schedule for $\epsilon_{i,t}^{(k)}$. For Term (a), which governs the decay of the initial potential, We establish its rate as follows:\begin{lemma}\label{lem:init_decay}Let $\eta_t = t^{-5/6}$ and $\tau_t = t^{-1/6}$. The initial error term decays asymptotically as:\begin{equation*}c_1\exp\left(-\sum_{j=1}^t\frac{\eta_{j}\tau_{j}}{2}\right) = \Theta(t^{-1/2}).
\end{equation*}
\end{lemma}
Accordingly, we set $\epsilon_{i,t}^{(1)}=\Theta(t^{-1/6})$. Since the margin decays significantly slower than the error ($t^{-1/6} \gg t^{-1/2}$), the difference becomes negative for large $t$. By the limit comparison test (\ref{lem:limit_comparison_bound}), the sum of the positive parts is $O(1)$. 

Terms (b) and (c) capture the optimization error and statistical error, respectively. Their asymptotic behavior are characterized below:
\begin{lemma}\label{lem: bound b and c}Let $\eta_t = t^{-5/6}$ and $\tau_t = t^{-1/6}$ for $t \ge 1$. For any $t \ge C''$, the per-episode error terms decay as: \begin{align*}    \left(\sum_{j=1}^t\eta_j^2\exp\left(-\sum_{k=j+1}^t \eta_k\tau_k\right)\right)^{1/2} &= \Theta(t^{-1/3}),\\   \left(\sum_{j=1}^t\eta_j\delta_j\exp\left(-\sum_{k=j+1}^t \eta_k\tau_k\right)\right)^{1/2} &\le \tilde{O}(t^{-1/6}). \end{align*} \end{lemma}
Since Term (b) decays as $\Theta(t^{-1/3})$, setting $\epsilon_{i,t}^{(2)}=\Theta(t)^{-1/6}$ ensures strict dominance, contributing $O(1)$ to the violation. Term (c) presents the tightest bottleneck, decaying at the same rate $\tilde{O}(t^{-1/6})$ as our regularization schedule. Consequently, we set $\epsilon_{i,t}^{(3)}$ to match this order but with sufficiently large logarithmic factors to cover the high-probability bound of $\delta_j$. Finally, for Term (d), the regularization bias is explicitly $\frac{4H}{\Xi}t^{-1/6}$. By setting $\epsilon_{i,t}^{(4)} \ge \frac{4H}{\Xi}t^{-1/6}$, we achieve complete covering of the bias term. Summing all components, we conclude that $\mathcal{R}_T(d)\le\tilde{O}(1)$. 

\paragraph{Strong Reward Regret.} Unlike in the violation analysis, the safety margin $\epsilon_{i,t}$ in the reward decomposition acts as additive penalties. Since both $\epsilon_{i,t}$ and $\tau_t$ are scheduled as $\Theta(t^{-1/6})$, their cumulative sums are bounded by $\sum_{i=1}^Tt^{-1/6}=\tilde{O}(T^{5/6})$. Furthermore, as established in Lemma \ref{lem: bound b and c}, the cumulative learning error is dominated by the statistical term, which also scales as $\tilde{O}(T^{5/6})$. Summing these components yields the final bound $\mathcal{R}_{T}(r) \le \tilde{O}(T^{5/6})$. Detailed proofs are provided in Appendix \ref{app:property} and \ref{sec:regrets}.


\section{Experiments}\label{sec:experiment}

We conduct experiments comparing our FlexDOME algorithm with the vanilla primal-dual baseline~\citep{efroni2020exploration} and the state-of-the-art (SOTA) UOpt-RPGPD algorithm~\citep{kitamura2024policy}. 
Our comparison focuses on algorithms that are most relevant to the last-iterate regime studied in this work. Targeted ablation studies are performed to dissect the contributions of our algorithm's key components. 
The evaluations are performed on randomly generated tabular CMDPs. Following the setup in \citet{kitamura2024policy}, we construct environments where the objective and the constraint are in conflict to create a non-trivial trade-off between reward maximization and violation minimization. We evaluate in two threshold settings: a stochastic-threshold setting and a standard fixed-threshold case. We set $S=20$, $A=H=5$ and focus on a single constraint for clear visualization. All results are averaged over 5 independent runs with different random seeds. The safety margin is scaled by a factor of $10^{-5}$ to mitigate the over-conservatism in theory and accelerate the algorithms' convergence. A comprehensive sensitivity analysis of the hyperparameters, along with other experimental details, is provided in Appendix \ref{app:experiment}.

Our empirical results fully corroborate our theoretical findings. Figure \ref{fig:1} shows that, in the stochastic-threshold environment, FlexDOME is the only algorithm that maintains near-zero instantaneous violation, leading to a flat, near-constant cumulative strong violation curve. In contrast, both the baseline and the SOTA method exhibit oscillatory behavior and incur growing strong constraint violation. The middle row of Figure \ref{fig:1} shows that FlexDOME retains its safety advantage in standard fixed-threshold environments; however, this robust constraint satisfaction comes at the cost of a slight trade-off in reward regret compared to UOpt-RPGPD. The ablation studies (bottom row) confirm that removing the regularization framework reintroduces the severe oscillations characteristic of standard primal-dual methods, underscoring its necessity for stable learning. FlexDOME closely tracks an oracle (with access to the true threshold), confirming that our estimation mechanism is efficient and does not compromise safety or performance.


\section{Conclusion}
This work provides an affirmative resolution to the fundamental trilemma among stringent safety, sublinear strong regret, and last-iterate convergence in online CMDPs. We propose FlexDOME, a primal-dual framework that employs a decaying safety margin with time-varying regularization to effectively navigate the safety-performance trade-off. We prove that FlexDOME can simultaneously achieve a $\tilde{O}(1)$ strong constraint violation, sublinear strong regret, and a non-asymptotic last-iterate convergence guarantee. To our best knowledge, FlexDOME is the first algorithm to achieve near-constant strong violation under the guarantee of last-iterate convergence. Our experiments corroborate these theoretical findings. 




\section*{Impact Statement}

This paper presents work whose goal is to advance the field of Machine
Learning. There are many potential societal consequences of our work, none
which we feel must be specifically highlighted here.


\bibliography{example_paper}
\bibliographystyle{icml2026}

\newpage
\appendix
\onecolumn
\section{CMDPs with Stochastic Thresholds}\label{app:distionction}

This section provides a rigorous analysis of the differences between the standard Constrained Markov Decision Process (CMDP) and the CMDPs with stochastic thresholds, as introduced in this work. We first formally define each setting, contrast their optimization objectives by highlighting the informational disparity, analyze for the non-degenerate nature of our problem formulation, and finally discuss the generality of our results.

\subsection{Formal Definitions of CMDPs with Stochastic Thresholds}

We begin by formally defining the two problem settings.

\begin{definition}[Standard CMDPs]
A standard episodic CMDP is defined by the tuple $\mathcal{M} = (\mathcal{S}, \mathcal{A}, H, p, r, d, \boldsymbol{\alpha})$, where $\mathcal{S}, \mathcal{A}, H, p, r, d$ are the state space, action space, horizon, transition dynamics, reward function, and constraint functions, respectively. The threshold $\boldsymbol{\alpha} = (\alpha_1, \dots, \alpha_m)$ is a vector of scalars, where each $\alpha_i \in \mathbb{R}$ is a \textbf{pre-specified and known constant} given as part of the problem definition.
\end{definition}

\begin{definition}[CMDPs with stochastic thresholds]
An episodic CMDP with stochastic thresholds is defined by the tuple $\mathcal{M}' = (\mathcal{S}, \mathcal{A}, H, p, r, d, \{\mathcal{L}_i\}_{i=1}^m)$, where the first six components are as defined above. For each constraint $i$, $\mathcal{L}_i$ is an unknown probability distribution from which the agent observes stochastic samples $\tilde{\alpha}_{i,h}^t\sim \mathcal{L}_i$ at each step $h$ and episode $t$. The expectation of these samples defines a threshold $\alpha_{i} = \mathbb{E}_{\mathcal{L}_i}[\tilde{\alpha}_{i,h}^t]$, where $\alpha_i$ is a scalar constant that is \textbf{unknown} to the agent.
\end{definition}

\begin{algorithm}[h]
\caption{Agent-Environment Interaction for $t \in [T]$}
\label{alg:interaction}
\begin{algorithmic}[1]
\STATE \textbf{Require:} Policy $\pi_t\in \Pi$
\STATE Environment initializes state $s_1 \in \mathcal{S}$
\FOR{$h = 1, \dots, H$}
    \STATE Agent takes action $a_h \sim \pi_t(\cdot \mid s_h)$
    \STATE Agent observes reward $\tilde{r}_h^t(s_h, a_h)$, constraint $\tilde{d}_{i,h}^t(s_h, a_h)$, and threshold $\tilde{\alpha}_{i,h}^t$ for $i \in [m]$
    \STATE Environment evolves to $s_{h+1} \sim p(\cdot \mid s_h, a_h)$
\ENDFOR
\end{algorithmic}
\end{algorithm}

Algorithm~\ref{alg:interaction} depicts this interaction, where at each episode $t$, the agent executes a policy $\pi_t$ and observes not only rewards and constraints, but also the thresholds themselves.

\subsection{Concentration of the Empirical Threshold Estimator}

We analyze the concentration properties of the empirical threshold estimator defined in Equation~\ref{eqs:empirical}. The following theorem establishes a high-probability bound on the deviation of this estimator from the true mean threshold $\alpha_i$.

\begin{lemma}[Concentration of empirical thresholds]\label{thm:concentration}
Assume the stochastic thresholds $\tilde{\alpha}_{i,h}^l$ are independently drawn for each episode $l\in[t]$ and step $h\in[H]$. Further, assume that each sample is bounded, such that $\tilde{\alpha}_{i,h}^l \in [0, H]$. Let the empirical estimator for the threshold of constraint $i$ at the beginning of episode $t+1$ be defined as
$$
\hat{\alpha}_{i}^{t+1} := \frac{1}{tH} \sum_{l=1}^t \sum_{h=1}^H \tilde{\alpha}_{i,h}^l,
$$
and let the true mean be $\alpha_i = \mathbb{E}[\tilde{\alpha}_{i,h}^{l}]$. Then, for any $\delta \in (0, 1)$, with probability at least $1-\delta$, the following bound holds:
$$
\left|\hat{\alpha}_{i}^{t+1} - \alpha_i\right| \le  \sqrt{\frac{H\log(2/\delta)}{2t}}:=\zeta_i^{t+1}.
$$
\end{lemma}

\begin{proof}
Let $\{X_j\}_{j=1}^n$ be a set of $n=tH$ independent random variables, where each $X_j$ corresponds to one of the observed stochastic thresholds $\tilde{\alpha}_{i,h}^l$ for $l \in [t], h \in [H]$. By assumption, each random variable is bounded within the interval $[0, H]$, thus for all $j$, the range $(b_j - a_j)$ is $H$.

The empirical estimator $\hat{\alpha}_{i}^{t+1}$ is the sample mean $\bar{X} = \frac{1}{n}\sum_{j=1}^n X_j$. The true mean $\alpha_i$ is the expected value of this sample mean, $\mathbb{E}[\bar{X}]$. By Hoeffding's inequality and Substituting our parameters ($n=tH$ and $b_j - a_j = H$), we have:
\begin{align*}
\mathbb{P}\left( \left|\hat{\alpha}_{i}^{t+1} - \alpha_i\right| \ge c \right) 
&\le 2 \exp\left( - \frac{2(tH)^2c^2}{\sum_{j=1}^{tH} H^2} \right) \\
&=  2 \exp\left( - \frac{2tc^2}{H} \right)
\end{align*}
We set the right-hand side of the probability bound: $\delta = 2 \exp\left( - \frac{2tc^2}{H} \right)$. Solving for the deviation $c$, we obtain
\begin{align*}
c &= \sqrt{\frac{H \log(2/\delta)}{2t}}.
\end{align*}
Thus, with probability at least $1-\delta$, the error $|\hat{\alpha}_{i}^{t+1} - \alpha_i|$ is bounded by $c$. This completes the proof.
\end{proof}

\begin{lemma}[Union bound for empirical thresholds]\label{lem:union-bound}
Given $\delta\in(0,1)$, with probability at least $1-\delta$, the following holds uniformly for each constraint $i\in[m]$ and episode $t\in[T]$:
$$
\left|\hat{\alpha}_{i}^{t+1}-\alpha_{i}\right|\leq\zeta^{t+1},
$$
where $\zeta^{t+1}=\sqrt{\frac{H\log(2mT/\delta)}{2t}}$.
\end{lemma}

\begin{proof}
By Lemma \ref{thm:concentration}, for any given confidence level $\delta'$ and given constraint $i$, we have:
\[
\mathbb{P}\left[\left|\hat{\alpha}_{i}^{t+1}-\alpha_{i}\right|\leq \sqrt{\frac{H\log(2/\delta')}{2t}}\right]\geq1-\delta'.
\]
Taking a union bound over all possible choices of $i\in[m]$ and $t\in[T]$, we have:
\begin{align*}
&\mathbb{P}\left[\bigcap_{i,t}\left\{\left|\hat{\alpha}_{i}^{t+1}-\alpha_{i}\right|\leq\zeta_i^{t+1}\right\}\right]\geq1-mT\delta'.
\end{align*}
Letting $\delta=mT\delta'$ and substituting into $\zeta_i^{t+1}$, we derive the stated uniform bound with probability at least $1-\delta$. This completes the proof.
\end{proof}

The theorem demonstrates that the empirical estimator $\hat{\alpha}_{i}^{t+1}$ converges to the true mean $\alpha_i$ at a rate of $\mathcal{O}(1/\sqrt{t})$.

\section{Preparation Lemmas}\label{app:preparation}

\begin{lemma}[\citet{muller2024truly}]\label{lem:last-iterate convergence-0}
Let $V := \Delta([d])$, and $g \in \mathbb{R}_{\ge 0}^d =: X$. Then $\tilde{x} := \arg\max_{x \in X} g^\top x - \frac{1}{\eta_{t}} \mathrm{KL}(\tilde{x}, x)$ and $\arg\max_{x \in V} g^\top x - \frac{1}{\eta_{t}} \mathrm{KL}(\tilde{x}, x)$ exist and are unique. Moreover, if $g$ only has non-negative entries, then for all $x^\star \in V$ we have
\begin{equation*}
    g^\top(x^\star - x) \le \frac{\mathrm{KL}(x^\star, x) - \mathrm{KL}(x^\star, x')}{\eta_{t}} + \frac{\eta_{t}}{2} \sum_{i=1}^d \tilde{x}_i g_i^2.
\end{equation*}
\end{lemma}

\begin{lemma}[\citet{muller2024truly}]
The performance gap admits the decomposition:
\begin{align*}
    &\quad V_{y_t}^{\pi_{\tau_{t},\epsilon}^{\star}}-V_{y_t}^{\pi_t}\\
    &=\hat{V}_{\bar{y}_t}^{t}-V_{y_t}^{\pi_t}\\
    &+\sum_{h\in[H]} \mathbb{E}\Bigl[\bigl\langle\hat{Q}_{\bar{y}_t,h}^t(s_h,\cdot),\pi_{\tau_{t},h}^\star(\cdot\mid s_h)-\pi_{t,h}(\cdot\mid s_h)\rangle\mid s_1, \pi_{\tau_{t},\epsilon}^{\star},p\Bigr]\\
    &+\sum_{h\in[H]} \mathbb{E}\Bigl[-\hat{Q}_{\bar{y}_t,h}^t(s_h,\cdot)+y_{t,h}(s_h,a_h)+\bigl\langle p_h(\cdot\mid s_h,a_h),\hat{V}_{\bar{y}_t,h+1}^t(\cdot)\bigr\rangle\mid s_1, \pi_{\tau_{t},\epsilon}^{\star},p\Bigr].
\end{align*}
\end{lemma}

\begin{lemma}[\citet{altman1999constrained}]\label{lem:mix policy}
    Suppose the transition function is $P$. For any mixed policy $\pi^{\text{mix}}=B_{\gamma}\pi^1+(1-B_{\gamma})\pi^2$, where $B_{\gamma}$ is a Bernoulli distributed random variable with mean $\gamma$. Then there exists a Markov policy $\hat{\pi}$ that
    \begin{equation*}
        V_{r,h}^{\hat{\pi}}(p)=V_{r,h}^{\pi^{\text{mix}}}(p),\qquad \forall r,s,h.
    \end{equation*}
\end{lemma}
\begin{lemma}\label{lem:limit_comparison_bound}
Let $\{A_t\}_{t=1}^\infty$ and $\{B_t\}_{t=1}^\infty$ be two sequences of positive real numbers. Assume that the limit of their ratio exists and is a constant $L$ strictly less than 1:
\[
\lim_{t\to\infty} \frac{A_t}{B_t} = L < 1.
\]
Then the partial sum $S_T = \sum_{t=1}^T [A_t - B_t]_+$ is bounded by a constant that is independent of $T$, i.e., $S_T = O(1)$, where $[x]_+ := \max(0, x)$.
\end{lemma}
\begin{proof}
To prove that the sum is $O(1)$, it must be shown that the corresponding infinite series $\sum_{t=1}^\infty [A_t - B_t]_+$ converges. For a series of non-negative terms, it is sufficient to show that the summand is identically zero for all terms beyond a finite threshold $t_0$. A non-zero term in the sum occurs only if $A_t > B_t$.

The given condition is $\lim_{t\to\infty} \frac{A_t}{B_t} = L$, where $L < 1$. By the formal definition of a limit, for every $\varepsilon > 0$, there exists a positive integer $t_0$ such that for all $t > t_0$, the inequality $\left|\frac{A_t}{B_t} - L\right| < \varepsilon$ holds. This is equivalent to:
\[
L - \varepsilon < \frac{A_t}{B_t} < L + \varepsilon.
\]

The objective is to prove that $\frac{A_t}{B_t} < 1$ for sufficiently large $t$. To achieve this from the inequality above, it is sufficient to ensure that the right-hand side, $L + \varepsilon$, is strictly less than 1. Since $L < 1$, the distance $1 - L$ is a fixed positive number. A valid and convenient choice for $\varepsilon$ is therefore:
\[
\varepsilon = \frac{1-L}{2}.
\]
This choice of $\varepsilon$ is guaranteed to be positive.

For this choice of $\varepsilon$, the right-hand side of the limit inequality becomes:
\[
L + \varepsilon = L + \frac{1-L}{2} = \frac{2L + 1 - L}{2} = \frac{1+L}{2}.
\]
Since $L < 1$, it follows that $1+L < 2$, and therefore $\frac{1+L}{2} < 1$.

By the definition of the limit, for the chosen $\varepsilon$, there must exist a threshold $t_0$ such that for all $t > t_0$:
\[
\frac{A_t}{B_t} < L + \varepsilon = \frac{1+L}{2}.
\]
As it has been shown that $\frac{1+L}{2} < 1$, it follows that for all $t > t_0$:
\[
\frac{A_t}{B_t} < 1.
\]

Since $B_t$ is a positive sequence, the inequality $\frac{A_t}{B_t} < 1$ implies $A_t < B_t$, which in turn means $A_t - B_t < 0$ for all $t > t_0$.
Therefore, the summand of the series becomes:
\[
[A_t - B_t]_+ = \max(0, A_t - B_t) = 0, \quad \forall t > t_0.
\]
The total sum can then be split into a finite part and a tail of zeros:
\[
\sum_{t=1}^T [A_t - B_t]_+ = \sum_{t=1}^{t_0} [A_t - B_t]_+ + \sum_{t=t_0+1}^{T} 0 = \sum_{t=1}^{t_0} [A_t - B_t]_+.
\]
This is a finite sum of finite numbers, which evaluates to a constant value that is independent of the upper limit $T$ (for $T > t_0$). Therefore, the sum is bounded by a constant, and the conclusion is that:
\[
\sum_{t=1}^T [A_t - B_t]_+ = O(1). \qedhere
\]
\end{proof}


\section{Feasibility and Strong Duality for the Margin-Regularized CMDP}\label{app:regularized-lagrangian}
Recall $\epsilon_{i,t}=18/5\sqrt{H^3C_B}\left(t^{-1/6}\cdot \log(SAHt/\delta')^{1/4}\right)$ for all constraint $i$, where $\delta'=\delta/4$ and $C_B=\big(1+\frac{8mH}{\Xi}\big)\big(4H\sqrt{2SA}\big(H\sqrt{S}+H+1\big)\big)+\frac{4mH}{\Xi}\sqrt{2H}$ . The existence of a feasible solution to (\ref{eq:opt-ques}) can be guaranteed if $\epsilon_{i,t}\le \Xi$. Let $C''$ be the smallest value such that for any $t\ge C''$, $\epsilon_{i,t}\le \Xi/2$. Then this optimization problem (\ref{eq:opt-ques}) has at least one feasible solution for any $t\ge C''$. By calculation, we can obtain $C''=O\left(K^3\cdot \log^{3/2} (K)\right)$, where $K=H^3C_B$ and then $C''$ is a constant and $T$-independent.

We establish the fundamental theoretical properties of the regularized Lagrangian formulation presented in Section \ref{sec:primal-dual scheme}. Our analysis proceeds by reformulating the problem in the space of occupancy measures. For clarity, we restate the regularized Lagrangian for a fixed episode $t$ and regularization parameter $\tau_{t} > 0$:
\begin{equation*}
    \mathcal{L}_{\tau_{t},t}(\pi, \lambda) := V_{r}^\pi(p) + \lambda^{\top}\left(V_{d}^\pi(p)-\epsilon_t-\alpha\right) + \tau_{t}\mathcal{H}(\pi) + \frac{\tau_{t}}{2} \|\lambda\|^2,
\end{equation*}
where the optimization problem is $\max_{\pi \in \Pi} \min_{\lambda \in \mathcal{C}} \mathcal{L}_{\tau_{t},t}(\pi, \lambda)$ over the policy space $\Pi$ and the compact dual domain $\mathcal{C} := [0, \frac{4H}{\Xi}]^m$.

\begin{lemma}[Strong duality of the margin-regularized problem]\label{lem:strong duality}
\textit{For any fixed episode $t\ge C''$ and regularization parameter $\tau_{t} > 0$, the regularized CMDP problem exhibits strong duality. That is,}
\begin{equation*}
    \max_{\pi \in \Pi} \min_{\lambda \in \mathcal{C}} \mathcal{L}_{\tau_{t},t}(\pi, \lambda) = \min_{\lambda \in \mathcal{C}} \max_{\pi \in \Pi} \mathcal{L}_{\tau_{t},t}(\pi, \lambda),
\end{equation*}
\textit{and both optima are attained.}
\end{lemma}

\begin{proof} Let $q^\pi \in \mathbb{R}^{HSA}$ be the occupancy measure corresponding to a policy $\pi \in \Pi$, defined as $q_h^\pi(s,a) := \mathbb{P}[s_h=s, a_h=a | s_1; p, \pi]$. The set of all valid occupancy measures forms a convex polytope, denoted by $Q(p)$. By definition, the police entropy is $\mathcal{H}(\pi)=-\mathbb{E}_{\pi}[\sum_h \log\pi_h(a_h|s_h)]$. The expectation can be rewritten as a sum over the state-action space: $\mathcal{H}(\pi)=-\sum_{h,s,a}q_h(s,a)\log\left(\frac{q_h(s,a)}{\sum_{a'}q_h(s,a')}\right):=\mathcal{H}(q)$. The value functions and policy entropy can be expressed as linear and strictly concave functions of $q^\pi$, respectively. We can thus define an equivalent Lagrangian over the domain $Q(p) \times \mathcal{C}$:
\begin{equation*}
    \bar{\mathcal{L}}_{\tau_{t},t}(q, \lambda) := r^\top q + \lambda(d^\top q - \epsilon_t - \alpha) + \tau_{t} \mathcal{H}(q) + \frac{\tau_{t}}{2} \|\lambda\|^2,
\end{equation*}
where $r, d \in \mathbb{R}^{HSA}$ are the vectors.

The optimization problem is equivalent to $\max_{q \in Q(p)} \min_{\lambda \in \mathcal{C}} \bar{\mathcal{L}}_{\tau_{t},t}(q, \lambda)$. We verify the conditions for Sion's Minimax Theorem \citep{sion1958general}:
\begin{enumerate}
    \item The domain $Q(p) \times \mathcal{C}$ is the product of a polytope and a hyperrectangle, and is therefore a non-empty, compact, and convex set.
    \item The function $\bar{\mathcal{L}}_{\tau_{t},t}(q, \lambda)$ is continuous over its domain.
    \item For any fixed $\lambda \in \mathcal{C}$, $\bar{\mathcal{L}}_{\tau_{t},t}(q, \lambda)$ is strictly concave in $q$. This is because $r^\top q + \lambda(d^\top q -\epsilon_t-\alpha)$ is linear in $q$, and the entropy term $\tau_{t} \mathcal{H}(q)$ is strictly concave for $\tau_{t} > 0$.
    \item For any fixed $q \in Q(p)$, $\bar{\mathcal{L}}_{\tau_{t},t}(q, \lambda)$ is strictly convex in $\lambda$. This is because $\lambda(d^\top q -\epsilon_t-\alpha)$ is linear in $\lambda$, and the term $\frac{\tau_{t}}{2}\|\lambda\|^2$ is strictly convex for $\tau_{t} > 0$.
\end{enumerate}
Since all conditions are met, it guarantees that the max-min and min-max values are equal and that optimizers exist.
\end{proof}

\begin{lemma}[Saddle point inequalities]\label{lem:SP ineq}
\textit{Let $(\pi_{\tau_{t},\epsilon}^{\star}, \lambda_{\tau_{t},\epsilon}^{\star})$ be the saddle point of $\mathcal{L}_{\tau_{t},t}$. Then for any episode $t\ge C''$, any policy $\pi \in \Pi$ and any dual variable $\lambda \in \mathcal{C}$, the following two inequalities hold:}
\begin{flalign*}
\text{(i)} \quad & V_{r}^\pi + \lambda_{\tau_{t},\epsilon}^{\star\top} (V_{d}^\pi -\epsilon_t-\alpha)+\tau_t\mathcal{H}(\pi) \le V_{r}^{\pi_{\tau_{t},\epsilon}^{\star}} + \lambda_{\tau_{t},\epsilon}^{\star\top} (V_{d}^{\pi_{\tau_{t},\epsilon}^{\star}} -\epsilon_t-\alpha) + \tau_{t} \mathcal{H}(\pi_{\tau_{t},\epsilon}^{\star}) \\
\text{(ii)} \quad & \lambda_{\tau_{t},\epsilon}^{\star\top} (V_{d}^{\pi_{\tau_{t},\epsilon}^{\star}} -\epsilon_t-\alpha) \le \lambda^\top (V_{d}^{\pi_{\tau_{t},\epsilon}^{\star}} -\epsilon_t-\alpha) + \frac{\tau_{t}}{2} (\|\lambda\|^2 - \|\lambda_{\tau_{t},\epsilon}^{\star}\|^2)&
\end{flalign*}
\end{lemma}

\begin{proof} According to Lemma \ref{lem:strong duality}, we immediately obtain $\mathcal{L}_{\tau_{t},t}(\pi, \lambda_{\tau_{t},\epsilon}^{\star}) \le \mathcal{L}_{\tau_{t},t}(\pi_{\tau_{t},\epsilon}^{\star}, \lambda_{\tau_{t},\epsilon}^{\star}) \le \mathcal{L}_{\tau_{t},t}(\pi_{\tau_{t},\epsilon}^{\star}, \lambda)$.
The inequalities are derived by expanding the saddle point definition from $\mathcal{L}_{\tau_{t},t}(\pi, \lambda_{\tau_{t},\epsilon}^{\star}) \le \mathcal{L}_{\tau_{t},t}(\pi_{\tau_{t},\epsilon}^{\star}, \lambda_{\tau_{t},\epsilon}^{\star})$ and $\mathcal{L}_{\tau_{t},t}(\pi_{\tau_{t},\epsilon}^{\star}, \lambda_{\tau_{t},\epsilon}^{\star}) \le \mathcal{L}_{\tau_{t},t}(\pi_{\tau_{t},\epsilon}^{\star}, \lambda)$, respectively.
\end{proof}


\section{Properties of the Model}\label{app:property}
\paragraph{Estimators} For each constraint $i\in[m]$, state $s$, action $a$, episode $t\in[T]$ and step $h\in[H]$, define $(s_h^l,a_h^l)$ as the state-action pair visited in episode $l$ at step $h$, and let $(s_h^l,a_h^l,s_{h+1}^l)$ denote the state-action pair $(s_h^l,a_h^l)$ is visited and the environment evolves to next state $s_{h+1}^l$ at step $h$ in episode $l$, let $\mathbf{1}_X$ represent the indicator function of $X$ and $N_h^{t}(s,a)=\sum_{l=1}^t\mathbf{1}_{\{s_h^l=s,\,a_h^l=a\}}$ is the total number of visits to the pair $(s,a)\in\mathcal{S}\times\mathcal{A}$ at step $h$ up to episode $t\in[T]$. We first give the empirical averages of the thresholds, rewards, constraints and transition probabilities as follows:
\begin{align*}
\hat{\alpha}_{i}^{t} &:= \frac{1}{tH} \sum_{l=1}^t \sum_{h=1}^H \tilde{\alpha}_{i,h}^l,\qquad\qquad\qquad\quad (\forall i\in[m])\\
    \hat{r}_h^t(s,a) &:= \frac{\sum_{l=1}^t \tilde{r}_{h}^l(s,a)\,\mathbf{1}_{\{s_h^l=s,a_h^l=a\}}}{\max\{1, N_h^{t}(s,a)\}}, \\
 \hat{d}_{i,h}^t(s,a) &:= \frac{\sum_{l=1}^t \tilde{d}_{i,h}^l(s,a)\,\mathbf{1}_{\{s_h^l=s,a_h^l=a\}}}{\max\{1, N_h^{t}(s,a)\}},\qquad (\forall i\in[m]) \\
 \hat{p}_h^t(s'\mid s,a) &:= \frac{\sum_{l=1}^t \mathbf{1}_{\{s_h^l=s,a_h^l=a,s_{h+1}^l=s'\}}}{\max\{1, N_h^{t}(s,a)\}}.
\end{align*}

Next, we define optimistic estimators for the reward, constraints and entropy bonus, and unbiased estimators for transition probabilities and thresholds as follows:
\begin{subequations}\label{opt-p}
\begin{align}
\overline{\alpha}_i^t&:=\hat{\alpha}_i^{t-1},\\
\overline{r}_h^t(s,a)&:=\hat{r}_h^{t-1}(s,a)+\phi_h^{t-1}(s,a),\\
\overline{d}_{i,h}^t(s,a)&:=\hat{d}_{i,h}^{t-1}(s,a)+\phi_h^{t-1}(s,a),\\
\overline{p}_h^t(s'|s,a)&:=\hat{p}_h^{t-1}(s'|s,a),\\
\overline{\psi}_{h}^t(s,a)&:=\psi_h^t(s,a)+\phi_h^{p,t-1}(s,a)\log(A).
\end{align}
\end{subequations}
The bonus term $\phi_h^t$ combines the uncertainties arising from both reward and transition estimations at step $h$ in episode $t$: $\phi_h^{t}(s,a) = \phi_{h}^{r,t}(s,a) + \phi_{h}^{p,t}(s,a)$, where the reward bonus $\phi_h^{r,t}(s,a)=\mathcal{O}\left(\sqrt{\frac{\text{ln}(mSAHT/\delta')}{\max\{1, N_h^t(s,a)\}}}\right)$ and the transition bonus $\phi_h^{p,t}(s,a)=\mathcal{O}\left(H\sqrt{\frac{S+\text{ln}(SAHT/\delta')}{\max\{1, N_h^t(s,a)\}}}\right)$ for any confidence parameter $\delta'\in (0,1)$. For convenience, we deonte
\begin{align*}
    y_t&:=r+\lambda_t^\top d+\tau_t\psi_t,\\
    \bar{y}_t&:=\bar{r}_t+\lambda_t^\top \bar{d}_t+\tau_t\bar{\psi}_t.
\end{align*}

\paragraph{Success event} Fixing a confidence parameter $\delta>0$ and defining $\delta':=\delta/4$, we first introduce the following \emph{failure events}:
\begin{align*}
   F_t^\alpha &:= \left\{ \exists i: \left| \hat{\alpha}_{i}^{t-1}- \alpha_{i} \right| \geq \zeta^{t-1}\right\},\\ 
   F_t^r &:= \left\{ \exists s, a, h : \left| \hat{r}_h^{t-1}(s,a) - r_h(s,a) \right| \geq \phi_{h}^{r,t-1}(s,a) \right\},\\
   F_t^d &:= \left\{ \exists s, a, h, i: \left| \hat{d}_{i,h}^{t-1}(s,a) - d_{i,h}(s,a) \right| \geq \phi_{i,h}^{r,t-1}(s,a) \right\},\\
   F_t^p &:= \left\{ \exists s, a,  h : \left\| p_h(\cdot\!\mid\!s,a) - \hat{p}_{h}^{t-1} (\cdot\!\mid\!s,a) \right\|_1 H \geq \phi_{h}^{p,t-1}(s,a) \right\},\\
   F_t^N &:= \left\{ \exists s, a, h : N_{h}^{t-1}(s,a) \leq \frac{1}{2} \sum_{j<t} q_h^{\pi_j}(s,a) - H \log \left(\frac{SAH}{\delta'}\right) \right\}.
\end{align*}
Then, we define the union of these events over all episodes,
\begin{align*}
    F^{\alpha}&:=\bigcup_{t\in [T]}F_t^\alpha, \quad
    F^r:=\left(\bigcup_{t\in [T]}F_t^r\right)\bigcup\left(\bigcup_{t\in [T]}F_t^d\right),\\[1ex]
    F^{p}&:=\bigcup_{t\in [T]}F_t^p,\quad
    F^{N}:=\bigcup_{t\in [T]}F_t^N.
\end{align*}
Furthermore, the success event $\mathcal{E}$ is defined as the complement of those failure events:
\begin{equation*}
    \mathcal{E} = \overline{F^\alpha\cup F^r \cup F^p \cup F^N}.
\end{equation*}
We have the following lemma.

\begin{lemma}[Success event]
Setting $\delta' = \frac{\delta}{4}$, we have $\mathbb{P}[\mathcal{E}] \geq 1-\delta$.
\end{lemma}

\begin{proof}
We apply the union bound to each event separately. By Lemma \ref{lem:union-bound}, we have $\mathbb{P}[F^{\alpha}] \leq \delta'$. Using Hoeffding’s inequality and union bound arguments over all state-action-step combinations, similarly, we obtain $\mathbb{P}[F^r] \leq \delta'$. Using concentration inequalities for multinomial distributions~\citep{maurer2009empirical} and the union bound, we derive $\mathbb{P}[F^p] \leq \delta'$. Employing similar techniques as in~\citep{dann2017unifying}, by bounding occupancy measure deviations, we obtain $\mathbb{P}[F^N] \leq \delta'$.

Combining these results with the union bound, we have
$$\mathbb{P}[F^\alpha\cup F^r \cup F^p \cup F^N]\leq \mathbb{P}[F^\alpha]+\mathbb{P}[F^r]+\mathbb{P}[F^p]+\mathbb{P}[F^N]\leq 4\delta'=\delta.$$
Thus, $\mathbb{P}[\mathcal{E}]=1-\mathbb{P}[F^\alpha\cup F^r \cup F^p \cup F^N]\geq 1-\delta$. 

This completes the proof.
\end{proof}

\paragraph{Truncated policy evaluation}
Truncated policy evaluation is essential in CMDPs 
under stochastic threshold settings. Given the presence of stochastic constraints and additional exploration bonuses, unbounded value estimates can lead to instability and hinder theoretical analysis. We employ truncation to maintain boundedness and numerical stability of value functions.

Formally, for given estimates of reward $\overline{r}_h(s,a)$, constraint functions $\overline{d}_{i,h}(s,a)$, transition probabilities $\overline{p}_h(\cdot\mid s,a)$, we iteratively compute truncated $Q$ and $V$ value estimates. The truncated Q-value update at each timestep $h$ is expressed as below,
\begin{equation*}
    \hat{Q}_h^{t}(s,a; \overline{l}, \overline{p}) = \min \left\{ \overline{l}_h(s,a) + \sum_{s'} \overline{p}_h(s'\mid s,a) \hat{V}_{\overline{l},h+1}^{t}(s'),\;H-h+1 \right\},
\end{equation*}
where $\overline{l}_h(s,a)$ denotes the generalized immediate payoff (reward or cost with bonus), and $\hat{V}_h^{\pi}(s; \overline{l}, \overline{p})$ denotes the truncated value function,
$$\hat{V}_{\overline{l},h}^{t}(s) = \Big\langle \hat{Q}_h^{t}(s,a; \overline{l}, \overline{p}), \pi_h^t(a \mid s) \Big\rangle.$$
For composite variable $\bar{y}_t$, we define its truncated value function as follows:
\begin{align*}
    \hat{Q}_{\bar{y}_t,h}^t(s, a) &:= \hat{Q}_{\bar{r}_t,h}^t(s, a) + \sum_{i=1}^m \lambda_{t,i} \hat{Q}_{\bar{d}_{i,t},h}^t(s, a)+\tau_t\hat{Q}_{\bar{\psi_t},h}^t(s,a),\\
    \hat{V}_{\bar{r}, h}^{t}(s) &= \langle \hat{Q}_h^t(s, \cdot; \bar{r}, \bar{p}), \pi_h^t(\cdot \mid s)\rangle.
\end{align*}
The detailed truncated policy evaluation algorithm is shown in Algorithm \ref{alg:main}.
\begin{algorithm}[t]
\caption{TPE (Truncated Policy Evaluation)}\label{alg:TPE}
\begin{algorithmic}[1]
\REQUIRE estimates $\bar{r}_h^t$, $\bar{d}_{i,h}^t$, $\bar{p}_h^t$, policy $\pi_h^t$.
\STATE Initial $\hat{V}_{\bar{r}, H+1}^{t}(s) = \hat{V}_{\bar{d_i}, H+1}^{t}(s)=\hat{V}_{\bar{\psi}, H+1}^{t}(s) = 0$ for all $s$, $i$.
\FOR{$h = H, H-1, \dots, 1$}
    \FOR{$(s, a) \in \mathcal{S} \times \mathcal{A}$}
        \STATE \textbf{Compute truncated Q-function:}
        \STATE $\hat{Q}_{\bar{r}_t,h}^t(s, a)= \min \left\{ \bar{r}_h(s, a) +\langle \bar{p}_h(\cdot \mid s, a) \hat{V}_{\bar{r}_t, h+1}^{t}(\cdot)\rangle, H - h + 1 \right\}$
        \STATE $\hat{Q}_{\bar{\psi}_t,h}^t(s,a) = \min \left\{ \bar{\psi}_h^t(s, a) + \langle \bar{p}_h(\cdot \mid s, a) \hat{V}_{\bar{\psi}_t, h+1}^{t}(\cdot)\rangle, \psi_h^t(s,a)+ (H - h + 1)\log(A) \right\}$
        \FOR{$i=1, \dots, m$}
            \STATE $\hat{Q}_{\bar{d}_{i,t},h}^t(s, a) = \min \left\{ \bar{d}_{i,h}(s, a) + \langle\bar{p}_h(\cdot \mid s, a) \hat{V}_{\bar{d}_{i,t}, h+1}^{t}(\cdot)\rangle, H - h + 1 \right\}$
        \ENDFOR
    \ENDFOR
    \FOR{all $s \in \mathcal{S}$}
        \STATE \textbf{Compute truncated V-function:}
        \STATE $\hat{V}_{\bar{r}, h}^{t}(s) = \langle \hat{Q}_{\bar{r}_t,h}^t(s, \cdot), \pi_h^t(\cdot \mid s)\rangle $
        \STATE $\hat{V}_{\bar{\psi}, h}^{t}(s) = \langle \hat{Q}_{\bar{\psi}_t,h}^t(s, \cdot), \pi_h^t(\cdot \mid s)\rangle $
        \FOR{$i=1, \dots, m$}
            \STATE $\hat{V}_{\bar{d}_i, h}^{t}(s) = \langle \hat{Q}_{\bar{d}_{i,t},h}^t(s, \cdot), \pi_h^t(\cdot \mid s) \rangle$
        \ENDFOR
    \ENDFOR
\ENDFOR
\FOR{$h= 1,\ldots,H$ and all $(s,a)$}
     \STATE $\hat{Q}_{\bar{y}_t,h}^t(s, a) := \hat{Q}_{\bar{r}_t,h}^t(s, a) + \sum_{i=1}^m \lambda_{t,i} \hat{Q}_{\bar{d}_{i,t},h}^t(s, a)+\tau_t\hat{Q}_{\bar{\psi}_t,h}^t(s,a)$
\ENDFOR
\STATE \textbf{Return} $\left\{ \hat{Q}_{\bar{y}_t,h}^t(s, a) \right\}_{h,s,a}$ and $\left\{ \hat{V}_{\bar{d}_{i,t}, h}^{t}(s) \right\}_{s,h,i}$
\end{algorithmic}
\end{algorithm}

\paragraph{Estimation error} We next show bounds on the estimation error of empirical estimator $\overline{r}$, $\overline{d}$ and $\overline{\psi}$.

\begin{lemma}[Bound on the estimation error]
\label{lem:explicit_estimation_error}
Let $T' \in [T]$ be a number of episodes. The total estimation error for the reward function and constraint function, conditioned on the success event $\mathcal{E}$, satisfies the following upper bound:
\begin{align*}
    \sum_{t=1}^{T'} (\hat{V}_{\bar{r}_t}^t - V_r^{\pi_t})&\le \left(2\sqrt{L_r} + 2H\sqrt{L_p}\right) \cdot \left(6HSA + 2H\sqrt{SAT'} + 2HSA\log(T') + 5\log\frac{2HT'}{\delta}\right),\\
    \sum_{t=1}^{T'} (\hat{V}_{\bar{d}_{t,i}}^t - V_{d_i}^{\pi_t})&\le \left(2\sqrt{L_r} + 2H\sqrt{L_p}\right) \cdot \left(6HSA + 2H\sqrt{SAT'} + 2HSA\log(T') + 5\log\frac{2HT'}{\delta}\right),\\
    \sum_{t=1}^{T'} (\hat{V}_{\bar{\psi}_{t}}^t - V_{\psi_t}^{\pi_t})&\le  
    2H\log(A)\sqrt{L_p} \left(6HSA + 2H\sqrt{SAT'} + 2HSA\log(T') + 5\log\frac{2HT'}{\delta}\right).
\end{align*}
where $L_r = \frac{1}{2}\log\left(\frac{2SAH(m+1)T}{\delta'}\right)$ and $L_p = 2S+2\log\left(\frac{SAHT}{\delta'}\right)$.
\end{lemma}

\begin{proof}
We first bound the total estimation error by the sum of the expectations of the bonus terms from \citet{muller2024truly}.
\begin{equation*}
    \sum_{t=1}^{T'} (\hat{V}_{\bar{r}_t}^t - V_r^{\pi_t}) \le 2\sum_{t=1}^{T'}\sum_{h=1}^{H}\mathbb{E}[\phi_{h}^{r,t-1}(s_{h}^t,a_{h}^t)] + 2\sum_{t=1}^{T'}\sum_{h=1}^{H}\mathbb{E}[\phi_{h}^{t-1,p}(s_{h}^t,a_{h}^t)].
\end{equation*}
By substituting the definitions of the bonus terms $b^r$ and $b^p$ and factoring out the shared summation structure, the total error is bounded by:
\begin{equation*}
    \sum_{t=1}^{T'} (\hat{V}_{\bar{r}_t}^t - V_r^{\pi_t}) \le \left(2\sqrt{L_r} + 2H\sqrt{L_p}\right) \sum_{t=1}^{T'}\sum_{h=1}^{H}\mathbb{E}\left[\frac{1}{\sqrt{n_{t-1,h}(s_h^t, a_h^t)\vee 1}}\right],
\end{equation*}
where $L_r = \frac{1}{2}\log\left(\frac{2SAH(m+1)T}{\delta'}\right)$ and $L_p = 2S+2\log\left(\frac{SAHT}{\delta'}\right)$. The core summation term involves the inverse square root of visitation counts. Using the high-probability bound for this term from 
\citet{liu2021learning}, it can be shown that:
\begin{equation*}
    \sum_{t=1}^{T'}\sum_{h=1}^{H}\mathbb{E}\left[\frac{1}{\sqrt{n_{t-1,h}(s_h^t, a_h^t)\vee 1}}\right] \le 6HSA + 2H\sqrt{SAT'} + 2HSA\log(T') + 5\log\frac{2HT'}{\delta}.
\end{equation*}
To put all term together, we get our final result:
\begin{equation*}
    \sum_{t=1}^{T'} (\hat{V}_{\bar{r}_t}^t - V_r^{\pi_t})\le \left(2\sqrt{L_r} + 2H\sqrt{L_p}\right) \cdot \left(6HSA + 2H\sqrt{SAT'} + 2HSA\log(T') + 5\log\frac{2HT'}{\delta}\right).
\end{equation*}
The proof for $d_i$ $(\forall i\in[m])$ is identical. For entropy bonus, we have
\begin{equation*}
    \sum_{t=1}^{T'} (\hat{V}_{\bar{\psi}_t}^t - V_{\psi_t}^{\pi_t}) \le  2\sum_{t=1}^{T'}\sum_{h=1}^{H}\mathbb{E}\left[\phi_{h}^{t-1,p}(s_{h}^t,a_{h}^t)\log(A)\right].
\end{equation*}
and the rest of the proof follows as in the proof of the case of reward function.
\end{proof}

\begin{lemma}[Bound on cumulative estimation discrepancies]\label{lem:bound on cumulative estimation}
Conditioned on the good event $\mathcal{E}$, for any episode $T'\in [T]$, the cumulative sum of the per-episode estimation discrepancies $\delta_t$ is bounded as follows:
\begin{equation*}
    \sum_{i=1}^{T'} \delta_i \le C_B\sqrt{T'\log\frac{SAHT}{\delta'}}+\tilde{O}(S^{3/2}AH^2).
\end{equation*}
where $C_B=\left(1+\frac{8mH}{\Xi}\right)\left(4H\sqrt{2SA}\left(H\sqrt{S}+H+1\right)\right)+\frac{4mH}{\Xi}\sqrt{2H}$ and $\delta_t$ is defined as the composite error at episode $t$:
$$\delta_t := \left(\hat{V}_{\bar{r}_t}^t-V_{r}^{\pi_t}\right) + \sum_{i=1}^m\lambda_{t,i}\left(\hat{V}_{\bar{d}_{t,i}}^t-V_{d_i}^{\pi_t}\right) + \tau_{t}\left(\hat{V}_{\bar{\psi}_{t}}^t-V_{\psi_t}^{\pi_t}\right) + \sum_{i=1}^m\lambda_{t,i}\left|\hat{\alpha}_i^t-\alpha_i\right|+\sum_{i=1}^m\frac{4H}{\Xi}\left(\hat{V}_{\bar{d}_{t,i}}^t-V_{d_i}^{\pi_t}\right).$$
\end{lemma}

\begin{proof}
The proof proceeds by decomposing the total sum $\sum_{t=1}^{T'}\delta_t$ and bounding each component term individually. Conditioned on the good event $\mathcal{E}$, we have:
\begin{align}
    \sum_{t=1}^{T'}\delta_t \le \underbrace{\sum_{t=1}^{T'}\left(\hat{V}_{\bar{r}_t}^t-V_{r}^{\pi_t}\right)}_{(A)} + \underbrace{\sum_{t=1}^{T'}\sum_{i=1}^m\frac{8H}{\Xi}\left(\hat{V}_{\bar{d}_{t,i}}^t-V_{d_i}^{\pi_t}\right)}_{(B)} + \underbrace{\sum_{t=1}^{T'}\tau_{t}\left(\hat{V}_{\bar{\psi}_{t}}^t-V_{\psi_t}^{\pi_t}\right)}_{(C)} + \underbrace{\sum_{t=1}^{T'}\sum_{i=1}^m\lambda_{t,i}\left|\hat{\alpha}_i^t-\alpha_i\right|}_{(D)} \label{eq:term-del-2}
\end{align}
We bound each of the four terms on the right-hand side of Equation~\eqref{eq:term-del-2}.

\textbf{Bounding terms (A), (B), and (C):} These terms represent the cumulative estimation errors for the value functions of the reward, constraints, and the policy entropy proxy $\psi_t$, respectively. We can bound them by leveraging Lemma~\ref{lem:explicit_estimation_error}.

For term (A), using Lemma~\ref{lem:explicit_estimation_error} and inequality $\sqrt{a+b}\le \sqrt{a}+\sqrt{b}$, it yields:
\begin{align*}
  (A) &\le \left(\sqrt{2\log\frac{2SAH(m+1)T}{\delta'}}+2H\sqrt{2S}+2H\sqrt{2\log{\frac{SAHT}{\delta'}}}\right)\cdot \left(2H\sqrt{SAT'}\right)+\tilde{O}(S^{3/2}AH^2)\\
  &\le \left(\sqrt{2\log(2(m+1))}+2H\sqrt{2S}+\left(1+2H\right)\sqrt{2\log{\frac{SAHT}{\delta'}}}\right)\cdot \left(2H\sqrt{SAT'}\right)+\tilde{O}(S^{3/2}AH^2)\\
  &\le \left(4H(H+1)\sqrt{2SA}+4H^2S\sqrt{2A}\right)\sqrt{T'\log\frac{SAHT}{\delta'}}+\tilde{O}(S^{3/2}AH^2)
\end{align*}
    
For term (B), we use the fact that the dual variables are bounded, i.e., $\lambda_{t,i} \le \left(\frac{4H}{\Xi}\right)$ for all $i \in [m]$. This allows us to write:
\begin{align*}
  (B) 
  &\le \frac{8H}{\Xi} \sum_{i=1}^m\left( \sum_{t=1}^{T'}\left(\hat{V}_{\bar{d}_{t,i}}^t-V_{d_i}^{\pi_t}\right) \right) \\
  &\le  \frac{8mH}{\Xi}\left(4H(H+1)\sqrt{2SA}+4H^2S\sqrt{2A}\right)\sqrt{T'\log\frac{SAHT}{\delta'}}+\tilde{O}(S^{3/2}AH^2).  
\end{align*}

For term (C), we apply the bound from Lemma~\ref{lem:explicit_estimation_error}:
\begin{align*}
    (C) &= \sum_{t=1}^{T'}\tau_{t}\left(\hat{V}_{\bar{\psi}_{t}}^t-V_{\psi_t}^{\pi_t}\right)\\
    &\le \max_t\{\tau_{t}\}\cdot\left(4H(H+1)\sqrt{2SA}+4H^2S\sqrt{2A}\right)\sqrt{T'\log\frac{SAHT}{\delta'}}+\tilde{O}(S^{3/2}AH^2)\\
    & \le\left(4H(H+1)\sqrt{2SA}+4H^2S\sqrt{2A}\right)\sqrt{T'\log\frac{SAHT}{\delta'}}+\tilde{O}(S^{3/2}AH^2)
\end{align*}

\textbf{Bounding term (D):} This term represents the cumulative error from the online estimation of the stochastic thresholds. Based on our derivation from Theorem 2, we have established a high-probability bound for this sum:
\begin{align*}
    (D) = \sum_{t=1}^{T'}\sum_{i=1}^m\lambda_{t,i}\left|\hat{\alpha}_i^t-\alpha_i\right| &\le m\left(\frac{4H}{\Xi}\right)\sqrt{2H\log(2mT/\delta)T'}.
\end{align*}

\textbf{Combining all terms:} By substituting the bounds for (A), (B), (C), and (D) back into Equation~\eqref{eq:term-del-2}, we obtain the final upper bound for the cumulative discrepancy:
\begin{equation*}
    \sum_{i=1}^{T'} \delta_i \le C_B\sqrt{T'\log\frac{SAHT}{\delta'}}+\tilde{O}(S^{3/2}AH^2)
\end{equation*}
where $C_B=\left(1+\frac{8mH}{\Xi}\right)\left(4H\sqrt{2SA}\left(H\sqrt{S}+H+1\right)\right)+\frac{4mH}{\Xi}\sqrt{2H}$. This completes the proof.
\end{proof}

\paragraph{$Q$-Value function bounds} We present the bounds for $Q$-value.
\begin{lemma}[\citet{muller2024truly}]\label{lem:term-bound}
For every state $s$, action $a$, step $h$, horizon $H$, and policy $\pi_t$ at $t$-th episode, it holds that
\begin{align*}
    \mathbb{E}\left[\sum_{h'=h}^H-\log\bigl(\pi_{t,h'}(a_{h'}\!\mid\!s_{h'})\bigr)\mid s_h=s, a_h=a\right]\le
     H\log (A)-\log(\pi_{t,h}(a\!\mid\!s)).
\end{align*}
\end{lemma}
\begin{lemma}[$Q$-Value function bounds]\label{lem:q-value-bounds}
For any state $s$, action $a$, step $h$, we get
\begin{equation*}
    0 \leq Q^{\pi_t}_{r + \lambda_t^\top d + \tau_{t} \psi_{t}, h}(s, a)
    \leq -\tau_{t} \log(\pi_{t,h}(a \mid s)) + H(1 +  \frac{4mH}{\Xi} + \tau_{t} \log(A))
\end{equation*}
Moreover, we have
\begin{align*}
    &\quad\sum_a \pi_{t,h}(a\!\mid\!s) \exp\left( \eta_{t} Q^{\pi_t}_{r + \lambda_t^\top d + \tau_{t} \psi_t, h}(s, a) \right)
    Q^{\pi_t}_{r + \lambda_t^\top d + \tau_{t} \psi_t, h}(s, a)^2\\
    &\leq \exp{\left(\eta_{t} H\left(1+\frac{4mH}{\Xi}+\tau_{t}\log(A)\right)\right)}\left(2A^{\eta_{t}\tau_{t}}H^2\left(1+\frac{4mH}{\Xi}+\tau_{t}\log(A)\right)^2 + \frac{128\tau_{t}^2\sqrt{A}}{e^2}\right).
\end{align*}
\end{lemma}
\begin{proof}
    We first prove the bound for $Q_{y_t}^{\pi_t}$, where $y_t=r+\lambda_t^\top d+\tau_{t} \psi_{t}$. For all $s$, $a$, $h$, we have
    \begin{align*}
        0\le Q^{\pi_t}_{r + \lambda_t^\top d + \tau_{t} \psi_{t}, h}(s, a) &\le \left|Q_{r,h}^{\pi_t}(s,a)\right|+\sum_{i}\lambda_{i,t}\left|Q_{d_{i,t},h}^{\pi_t}\right|+\tau_{t}\left|Q_{\psi_{t},h}^{\pi_t}\right|\\
        &\le H+m\left(\frac{4H}{\Xi}\right)H+\tau_{t}\mathbb{E}\left[\sum_{h'=h}^H-\log\bigl(\pi_{t,h'}(a_{h'}\!\mid\!s_{h'})\bigr)\mid s_h=s, a_h=a\right]\\
        &\le \underbrace{H\left(1+m\left(\frac{4H}{\Xi}\right)+\tau_{t}\log(A)\right)}_{C_0}-\tau_{t}\log(\pi_{t,h}(a\!\mid\!s)).
    \end{align*} 
The last inequality holds by Lemma \ref{lem:term-bound}. According to the definitions, we have $y_{t,h}(s,a)=r_{t,h}(s,a)+\lambda_t^{\top}d_{t,h}(s,a)+\tau_{t}\psi_{t,h}(s,a)$, where $\psi_{t,h}(s,a)=-\log(\pi_{t,h}(a\!\mid\!s))$. Moreover, according to Euclidean triangle inequality $(a+b)^2\le 2a^2+2b^2$, we can obtain
\begin{equation*}
    Q_{y}^{\pi_t}(s,a)^2\le\underbrace{2H^2(1+m\left(\frac{4H}{\Xi}\right)+\tau_{t}\log(A))^2}_{C_1}+2\tau_{t}^2\log^2(\frac{1}{\pi_{t,h}(a\!\mid\!s)}).
\end{equation*}
We then get the following inequality:
\begin{align*}
    \sum_a \pi_{t,h}(a\!\mid\!s) \exp{\left( \eta_{t} Q^{\pi_t}_{y_t}(s,a)\right)}
    Q^{\pi_t}_{y_t}(s, a)^2\leq&\underbrace{\sum_a \pi_{t,h}(a\!\mid\!s) \exp{\left( \eta_{t} Q^{\pi_t}_{y_t}(s,a)\right)}C_1}_{(1)}\\
    &+\underbrace{\sum_a \pi_{t,h}(a\!\mid\!s) \exp{\left(\eta_{t} Q^{\pi_t}_{y_t}(s,a)\right)}2\tau_{t}^2\log^2(\frac{1}{\pi_{t,h}(a\!\mid\!s)})}_{(2)}.
\end{align*}
For term (1) on the right-side, we first show
\begin{align*}
    \pi_{t,h}(a\!\mid\!s) \exp{\left( \eta_{t} Q^{\pi_t}_{y_t}(s,a)\right)}&\leq\pi_{t,h}(a\!\mid\!s)\exp{\left( \eta_{t} C_0-\eta_{t}\tau_{t}\log(\pi_{t,h}(a\!\mid\!s))\right)}\\
    &=\pi_{t,h}(a\!\mid\!s)^{1-\eta_{t}\tau_{t}} \exp{(\eta_{t} C_0)}.
\end{align*}
Thus, we obatin
\begin{align}\label{eq:term(1)}
    (1)&\leq\sum_a C_1\pi_{t,h}(a\!\mid\!s)^{1-\eta_{t}\tau_{t}} \exp{\left( \eta_{t} C_0\right)}\nonumber\\
    &\le  A^{\eta_{t}\tau_{t}}\exp(\eta_{t} C_0)C_1\nonumber\\
    &=A^{\eta_{t}\tau_{t}}\exp\left(\eta_{t} H\left(1+\frac{4mH}{\Xi}+\tau_{t}\log(A)\right)\right)\left(2H^2\left(1+\frac{4mH}{\Xi}+\tau_{t}\log(A)\right)^2\right).
\end{align}
The second inequality holds because 
$\sum_a \pi_{t,h}(a\!\mid\!s)^{1-\eta_{t}\tau_{t}}\le \max_{\pi}\sum_{a}\pi_{t,h}(a\!\mid\!s)^{1-\eta_{t}\tau_{t}}\le A^{\eta_{t}\tau_{t}}$. This is because the extreme case is the uniform distribution. Furthermore, for term (2), we follow the analysis in \citet{muller2024truly}. Assuming $\eta_t\tau_t \le 1/2$, we have $\pi_{t,h}(a|s)^{1-\eta_t\tau_t} \le \pi_{t,h}(a|s)^{1/2}$. We then use the property that $q^{1/4}\log^2(1/q)$ is universally bounded by $64/e^2$ for any $q$ and apply the Cauchy-Schwarz inequality to the remaining sum. This yields:
\begin{align}\label{eq:term(2)}
(2)&\le \sum_{a}\pi_{t,h}(a\!\mid\!s)^{1-\eta_{t}\tau_{t}}\exp(\eta_{t} C_0)2\tau_{t}^2\log^2\left(\frac{1}{\pi_{t,h}(a\!\mid\!s)}\right)\nonumber\\
&=\exp(\eta_{t} C_0)2\tau_{t}^2\sum_{a}\pi_{t,h}(a\!\mid\!s)^{1-\eta_{t}\tau_{t}}\log^2\left(\frac{1}{\pi_{t,h}(a\!\mid\!s)}\right)\nonumber\\
&\le\exp(\eta_{t} C_0)2\tau_{t}^2\left(\frac{64\sqrt{A}}{e^2}\right)\nonumber\\
&=\frac{128\tau_{t}^{2}\sqrt{A}}{e^2} \exp{(\eta_{t} C_0)}.
\end{align}
Combining \eqref{eq:term(1)} and \eqref{eq:term(2)}, we have
\begin{align*}
    &\sum_a \pi_{t,h}(a\!\mid\!s) \exp{\left( \eta_{t} Q^{\pi_t}_{y_t}(s,a)\right)}
    Q^{\pi_t}_{y_t}(s, a)^2\\
    &\leq \exp{\left(\eta_{t} H\left(1+\frac{4mH}{\Xi}+\tau_{t}\log(A)\right)\right)}\left(2A^{\eta_{t}\tau_{t}}H^2\left(1+\frac{4mH}{\Xi}+\tau_{t}\log(A)\right)^2 + \frac{128\tau_{t}^2\sqrt{A}}{e^2}\right).
\end{align*}
\end{proof}
The bounds established for $Q$ also apply to the truncated $\hat{Q}$. This is because $\hat{Q}_{y_t,h}\le H-h+1$ by definition, and is less than or equal to the initial bound $C_0-\tau_{t}\log(\pi_{t,h}(a|s))$ used in the proof above. We express the result as follows.
\begin{lemma}[$Q$-Value function bounds]
For any state $s$, action $a$, step $h$, we get
\begin{equation*}
    0 \leq \hat{Q}^{t}_{\bar{y}_{t}, h}(s, a)
    \leq -\tau_{t} \log(\pi_{t,h}(a \mid s)) + H(1 +  \frac{4mH}{\Xi} + \tau_{t} \log(A)).
\end{equation*}
Moreover, we have
\begin{align*}
    &\quad\sum_a \pi_{t,h}(a\!\mid\!s) \exp\left( \eta_{t} \hat{Q}^{t}_{\bar{y}_{t}, h}(s, a) \right)
    \hat{Q}^{t}_{\bar{y}_{t}, h}(s, a)^2\\
    &\leq \exp{\left(\eta_{t} H\left(1+\frac{4mH}{\Xi}+\tau_{t}\log(A)\right)\right)}\left(2A^{\eta_{t}\tau_{t}}H^2\left(1+\frac{4mH}{\Xi}+\tau_{t}\log(A)\right)^2 + \frac{128\tau_{t}^2\sqrt{A}}{e^2}\right).
\end{align*}
\end{lemma}

\paragraph{Error summation bounds} To establish our main regret bounds, we need carefully analyze the cumulative effect of two primary sources of error that arise from our learning algorithm: the optimization error stemming from the primal-dual updates, and the statistical error resulting from estimating the unknown CMDP model. The following two lemmas provide crucial bounds on summations that capture the behavior of these error terms over time.

\begin{lemma}[Bound on the optimization error]\label{lem: bound b}
Let $\eta_t = t^{-5/6}$ and $\tau_t = t^{-1/6}$ for $t \ge 1$, and $C_0 = \sqrt{HC+D}$. Then for any $T \ge C''$, we have:
\begin{flalign*}
&\text{(i) Per-episode error decay:} \quad \left(\sum_{j=1}^t\eta_j^2\exp\left(-\sum_{k=j+1}^t \eta_k\tau_k\right)\right)^{1/2} = \Theta(t^{-1/3}),& \\
&\text{(ii) Cumulative error bound:} \quad \sum_{t=C''}^T H^{3/2}C_0 \left(\sum_{j=1}^t\eta_j^2\exp\left(-\sum_{k=j+1}^t \eta_k\tau_k\right)\right)^{1/2} \le K \cdot T^{2/3},
\end{flalign*}
where $K = \frac{3\sqrt{7}}{2^{2/3}} H^{3/2} \sqrt{HC+D}$.
\end{lemma}\begin{proof}Let $S$ denote the sum. We factor out the constants and define the inner term $X_t$
\begin{equation*}
S = H^{3/2} C_0 \sum_{t=C''}^T X_t^{1/2}, \quad \text{where} \quad X_t = \sum_{j=1}^t j^{-5/3}\exp\left(-\sum_{k=j+1}^t k^{-1}\right).
\end{equation*}
To bound $X_t$, we split the sum over $j$ into two parts: $j \in [1, \lfloor t/2 \rfloor]$ and $j \in [\lfloor t/2 \rfloor+1, t]$.

For $j \in [1, \lfloor t/2 \rfloor]$, the sum in the exponent is large. We can bound it from below by integrating over the range $[j+1, t+1]$: $\sum_{k=j+1}^t k^{-1} \ge \int_{j+1}^{t+1} x^{-1} dx = \ln\left(\frac{t+1}{j+1}\right)$. This implies the exponential term is bounded by $\frac{j+1}{t+1}$. The sum is thus bounded by:
\begin{equation*}
\sum_{j=1}^{\lfloor t/2 \rfloor} j^{-5/3} \left(\frac{j+1}{t+1}\right) \le \frac{2}{t} \sum_{j=1}^{\lfloor t/2 \rfloor} j^{-2/3} \le \frac{2}{t} \cdot 3(t/2)^{1/3} = c_1 t^{-2/3}.
\end{equation*}
This term decays polynomially with $t$.

For $j \in [\lfloor t/2 \rfloor+1, t]$, we have $j^{-5/3} \le (t/2)^{-5/3} = 2^{5/3}t^{-5/3}$. We find a lower bound for the exponent's sum: $\sum_{k=j+1}^t k^{-1} \ge (t-j)t^{-1}$. This gives the bound on the sum for this part:
\begin{equation*}
\sum_{j=\lfloor t/2 \rfloor+1}^t 2^{5/3}t^{-5/3} \exp\left(-(t-j)t^{-1}\right).
\end{equation*}
Letting $l=t-j$, this becomes $2^{5/3}t^{-5/3} \sum_{l=0}^{\lceil t/2 \rceil-1} (e^{-t^{-1}})^l$. For large $t$, $t^{-1}$ is small. Using the inequality $1-e^{-x} \ge x/2$ for sufficiently small $x > 0$, we bound the sum of the series by $\frac{1}{1-e^{-t^{-1}}} \le \frac{1}{(1/2)t^{-1}} = 2t$. The bound for this second part is therefore:
\begin{equation*}
\sum_{j=\lfloor t/2 \rfloor+1}^t  j^{-5/3}\exp\left(-\sum_{k=j+1}^t k^{-1}\right)\le 2^{5/3}t^{-5/3} \cdot 2t = 2^{8/3} \cdot t^{-2/3}.
\end{equation*}
Combining the two parts, $X_t \le c_1 t^{-2/3} + 2^{8/3} t^{-2/3}$. Since both terms decay at the same rate, for $t \ge C''$ we can define a constant $K_X=7 \cdot 2^{2/3}$ such that $X_t \le K_X t^{-2/3}$. Consequently, $X_t^{1/2} \le \sqrt{K_X} t^{-1/3}$.Substituting this back into the expression for $S$:
\begin{equation*}
S \le H^{3/2} C_0 \sqrt{K_X} \sum_{t=C''}^T t^{-1/3}.
\end{equation*}
The sum is for a p-series with $p = 1/3$, which we bound with an integral:
\begin{equation*}
\sum_{t=C''}^T t^{-1/3} \le \int_{0}^T x^{-1/3} dx = \frac{3}{2}T^{2/3}.
\end{equation*}

Combining all terms yields the final bound:
\begin{equation*}
S \le H^{3/2} C_0 \sqrt{K_X} \left(\frac{3}{2}T^{2/3}\right) = \frac{3}{2} H^{3/2} \sqrt{HC+D} \sqrt{K_X} \cdot T^{2/3}.
\end{equation*}
\end{proof}

\begin{lemma}[Bound on the statistical error]\label{lem: bound_c}Let $\eta_t = t^{-5/6}$ and $\tau_t = t^{-1/6}$ for $t \ge 1$. Let $\{\delta_t\}_{t=1}^T$ be a sequence whose partial sums are bounded by $\sum_{i=1}^{T'} \delta_i \le C_B\sqrt{T'\log\frac{SAHT}{\delta'}}+\tilde{O}(S^{3/2}AH^2):=B(T')$, where $C_B$ is defined as before. Then for any $t \ge C''$, the following inequality holds:
\begin{flalign*}
&\text{(i) Per-episode error decay:} \quad \left(\sum_{j=1}^t\eta_j\delta_j\exp\left(-\sum_{k=j+1}^t \eta_k\tau_k\right)\right)^{1/2} \le \tilde{O}(t^{-1/6}), & \\
&\text{(ii) Cumulative error bound:} \quad \sum_{t=C''}^T \sqrt{2H^3}\left(\sum_{j=1}^t\eta_j\delta_j\exp\left(-\sum_{k=j+1}^t \eta_k\tau_k\right)\right)^{1/2} \le \tilde{O}(T^{5/6}). &
\end{flalign*}
\end{lemma}

\begin{proof}Let $S_1$ denote the sum. We can write it as $S_1 = \sqrt{2H^3} \sum_{t=C''}^T \sqrt{|Y_t|}$, where the inner term $Y_t$ is defined as:
\begin{equation*}
Y_t = \sum_{j=1}^t \eta_j \delta_j \exp\left(-\sum_{k=j+1}^t \eta_k\tau_k\right).
\end{equation*}
The presence of the sequence $\delta_j$, for which we only have a bound on its cumulative sum, necessitates the use of summation by parts. Let $f_{t,j} = \eta_j \exp\left(-\sum_{k=j+1}^t \eta_k\tau_k\right)$ and $\Delta_j = \sum_{i=1}^j \delta_i$. A critical property for this method is the monotonicity of $f_{t,j}$ with respect to $j$. We examine the ratio:$$\frac{f_{t,j+1}}{f_{t,j}} = \frac{\eta_{j+1}}{\eta_j} \exp(\eta_{j+1}\tau_{j+1}).$$Based on the choice of $\eta_j$ and $\tau_j$, this ratio is greater than 1 for all $j \ge 1$. Thus, $f_{t,j}$ is strictly increasing in $j$. Applying the summation by parts formula to $Y_t = \sum_{j=1}^t f_{t,j}(\Delta_j - \Delta_{j-1})$ with $\Delta_0=0$:$$Y_t = f_{t,t}\Delta_t - \sum_{j=1}^{t-1} \Delta_j (f_{t,j+1}-f_{t,j}).$$Using the triangle inequality and the established monotonicity $f_{t,j+1}-f_{t,j} \ge 0$:$$|Y_t| \le |f_{t,t}||\Delta_t| + \sum_{j=1}^{t-1} |\Delta_j| (f_{t,j+1}-f_{t,j}).$$We use the given bound $|\Delta_j| \le B(j)$. Since $B(j)$ is an increasing function of $j$, we can bound $|\Delta_j| \le B(t-1)$ for all terms inside the summation. The sum is a telescoping series equal to $f_{t,t} - f_{t,1}$. Since $B(t-1) < B(t)$ and $f_{t,1} > 0$:$$|Y_t| \le f_{t,t}B(t) + B(t-1)(f_{t,t} - f_{t,1}) \le f_{t,t}B(t) + B(t)f_{t,t} = 2f_{t,t}B(t).$$For $t \ge C''$, we have $f_{t,t}=\eta_t = t^{-5/6}$. Substituting the form for the bound $B(t)$:
\begin{equation*}
|Y_t| \le 2t^{-5/6} B(t) = 2C_Bt^{-1/3}\sqrt{\log\frac{SAHT}{\delta'}}+\tilde{O}(S^{3/2}AH^2)\cdot t^{-5/6}.
\end{equation*}
Taking the square root and factoring out the dominant term and using inequality $\sqrt{a+b}\le \sqrt{a}+\sqrt{b}$, we obtain a bound for $\sqrt{|Y_t|}$:
\begin{align*}
\sqrt{|Y_t|} &\le \sqrt{2C_{B}t^{-1/3}\sqrt{\log\frac{SAHT}{\delta'}} + \tilde{O}(S^{3/2}AH^2)\cdot t^{-5/6}} \\&\le \sqrt{2C_{B}t^{-1/3}\sqrt{\log\frac{SAHT}{\delta'}}} + \sqrt{\tilde{O}(S^{3/2}AH^2)t^{-5/6}} \\&= \sqrt{2C_{B}} \cdot t^{-1/6} \left(\log\frac{SAHT}{\delta'}\right)^{1/4} + \tilde{O}(S^{3/4}A^{1/2}H) \cdot t^{-5/12}.
\end{align*}
The main sum is thus bounded by:
\begin{align*}
S_{1} &\le \sqrt{8H^3C_{B}}\left(\log\frac{SAHT}{\delta'}\right)^{1/4}\cdot\sum_{t=C''}^Tt^{-1/6}  + \tilde{O}(S^{3/4}A^{1/2}H) \cdot\sum_{t=C''}^T t^{-5/12},\\&\le \frac{12}{5}\sqrt{2H^3C_B}\left(\log\frac{SAHT}{\delta'}\right)^{1/4}\cdot T^{5/6}+\tilde{O}(S^{3/4}A^{1/2}H) \cdot T^{7/12},\\&\le \tilde{O}(T^{5/6}).
\end{align*}
\end{proof}

\section{Regrets of Reward and Constraint Violations}\label{sec:regrets}

In this section, we prove the bounds for the strong reward regret and constraint violation of Algorithm \ref{alg:main}. We first establish the linear convergence of the primal-dual divergence potential functions.
\begin{customlem}{1}[Margin-regularized convergence]\label{lem:regularized convergence}
Let $\eta_{t}$, $\tau_{t}<1$ and a confidence parameter $\delta\in(0,1)$. With probability at least $1-\delta$, the policy-dual divergence potential of Algorithm \ref{alg:main} holds
\begin{align*}
   \Phi_{t+1}\le \exp\left(-\sum_{j=1}^t\eta_j\tau_j \right)\Phi_1+\frac{HC+D}{2}\sum_{j=1}^t\eta_j^2\exp\left(-\sum_{k=j+1}^t\eta_k\tau_k\right)+\sum_{j=1}^t\eta_j\delta_j\exp\left(-\sum_{k=j+1}^t\eta_k\tau_k\right).
\end{align*}
where $C=\exp{\left(\eta_{t} H\left(1+\frac{4mH}{\Xi}+\tau_{t}\log(A)\right)\right)}\left(2A^{\eta_{t}\tau_{t}}H^2\left(1+\frac{4mH}{\Xi}+\tau_{t}\log(A)\right)^2 + \frac{128\tau_{t}^2\sqrt{A}}{e^2}\right)$,\\  $D=m\bigl(H + \tau_{t}\,\left(\frac{4H}{\Xi}\right)\bigr)^2$ and $\delta_j=\hat{V}_{\bar{y}_j}^j-V_{y_j}^{\pi_j}+\sum_{i}\frac{4H}{\Xi}\left(\hat{V}_{\bar{d}_{j,i}}^j-V_{d_i}^{\pi_j}\right)$.
\end{customlem}
\begin{proof}
    Conditioned on success event $\mathcal{E}$, we first decompose the primal dual gap for every episode $t$:
    \begin{equation*}
    \mathcal{L}_{\tau_{t},t}(\pi^\star_{\tau_{t},\epsilon}, \lambda_t) - \mathcal{L}_{\tau_{t},t}(\pi_t, \lambda^\star_{\tau_{t},\epsilon}) = 
    \underbrace{\mathcal{L}_{\tau_{t},t}(\pi^\star_{\tau_{t},\epsilon}, \lambda_t) - \mathcal{L}_{\tau_{t},t}(\pi_t, \lambda_t)}_{\text{(1)}}
    + \underbrace{\mathcal{L}_{\tau_{t},t}(\pi_t, \lambda_t) - \mathcal{L}_{\tau_{t},t}(\pi_t, \lambda^\star_{\tau_{t},\epsilon})}_{\text{(2)}}.
\end{equation*}
\paragraph{Bounding term (1)} For term (1), by Lemmas \ref{lem:last-iterate convergence-0} and \ref{lem:q-value-bounds}, we have
\begin{align}\label{eq:term(1)-1}
    (\text{1})
&=\mathcal{L}_{\tau_{t},t}\bigl(\pi_{\tau_{t,\epsilon}}^\star,\lambda_t\bigr)-\mathcal{L}_{\tau_{t},t}\bigl(\pi_{t},\lambda_t\bigr)\nonumber
\\
&=V_{r+\lambda_t^Tg}^{\pi_{\tau_{t,\epsilon}}^\star}-
V_{r+\lambda_t^Tg}^{\pi_{t}}\nonumber\tag{$g=d-\frac{1}{H}\alpha$}\\
&\quad+\tau_{t}\sum_{s,a,h}q_{h}^{\pi_{t}}(s)\pi_{t,h}(a\!\mid\!s)\log\bigl(\pi_{t,h}(a\!\mid\!s)\bigr)-
\tau_{t}\sum_{s,a,h}q_{h}^{\pi_{\tau_{t,\epsilon}}^\star}(s)\pi_{\tau_{t},\epsilon,h}^\star(a\!\mid\!s)\,\log\bigl(\pi_{\tau_{t},\epsilon,h}^\star(a\!\mid\!s)\bigr)\nonumber\\
&=V_{r+\lambda_t^Tg+\tau_{t} \psi_t}^{\pi_{\tau_{t,\epsilon}}^\star}-
V_{r+\lambda_t^Tg+\tau_{t}\psi_t}^{\pi_{t}}\nonumber\\
&\quad+\tau_{t}\sum_{s,a,h}q_{h}^{\pi_{\tau_{t,\epsilon}}^\star}(s)\pi_{\tau_{t},\epsilon,h}^\star(a\!\mid\!s)\log\bigl(\pi_{t,h}(a\!\mid\!s)\bigr)-
\tau_{t}\sum_{s,a,h}q_{h}^{\pi_{\tau_{t,\epsilon}}^\star}(s)\pi_{\tau_{t},\epsilon,h}^\star(a\!\mid\!s)\,\log\bigl(\pi_{\tau_{t},\epsilon,h}^\star(a\!\mid\!s)\bigr)\nonumber\\
&=V_{r+\lambda_t^Tg+\tau_{t} \psi_t}^{\pi_{\tau_{t,\epsilon}}^\star}-
V_{r+\lambda_t^Tg+\tau_{t}\psi_t}^{\pi_{t}}-\tau_{t} \textup{KL}_t\nonumber\tag{$\mathrm{KL_t}=\sum_{s,a,h}q_{h}^{\pi_{\tau_{t,\epsilon}}^\star}(s)\pi_{\tau_{t},\epsilon,h}^\star(a\!\mid\!s)\log\left(\frac{\pi^\star_{\tau_t,\epsilon,h}(a\mid s)}{\pi_{t,h}(a\mid s)}\right)$}\\
&=V_{y_t}^{\pi_{\tau_{t,\epsilon}}^\star}-V_{y_t}^{\pi_t}-\tau_{t} \textup{KL}_t\nonumber\\
    &\le (\hat{V}_{\bar{y}_t}^t-V_{y_t}^{\pi_t})+\frac{\mathrm{KL}_{t}-\mathrm{KL}_{t+1}}{\eta_{t}}+\frac{\eta_{t} H}{2}C-\tau_{t}\mathrm{KL}_{t}\nonumber\\
    &\le (\hat{V}_{\bar{y}_t}^t-V_{y_t}^{\pi_t})+\frac{(1-\eta_{t}\tau_{t})\mathrm{KL}_{t}-\mathrm{KL}_{t+1}}{\eta_{t}}+\frac{\eta_{t} H}{2}C,
\end{align}
where $C=\exp{\left(\eta_{t} H\left(1+\frac{4mH}{\Xi}+\tau_{t}\log(A)\right)\right)}\left(2A^{\eta_{t}\tau_{t}}H^2\left(1+\frac{4mH}{\Xi}+\tau_{t}\log(A)\right)^2 + \frac{128\tau_{t}^2\sqrt{A}}{e^2}\right)$.
\paragraph{Bounding term (2)} 
\begin{align}\label{eq:term(1)-2}
    (2)&=\sum_{i}(\lambda_{t,i}-\lambda_{\tau_{t},\epsilon,i}^{\star})(\hat{V}_{\bar{d}_{t,i}}^{t}-\alpha_i-\epsilon_t+\tau_{t}\lambda_{t,i})+\sum_{i}(\lambda_{t,i}-\lambda_{\tau_{t},\epsilon,i}^{\star})(V_{d_i}^{\pi_t}-\hat{V}_{\bar{d}_{t,i}}^t)-\frac{\tau_{t}}{2}\|\lambda_t-\lambda_{\tau_{t},\epsilon}^{\star}\|^2\nonumber\\
    &\le \frac{\|\lambda_{\tau_{t},\epsilon}^{\star}-\lambda_t\|^2-\|\lambda_{\tau_{t},\epsilon}^{\star}-\lambda_{t+1}\|^2}{2\eta_{t}}+\frac{\eta_{t}}{2}\|\hat{V}_{\bar{d}}^t-\alpha-\epsilon_t+\tau_{t}\lambda_t\|^2\nonumber\\
    &\quad+\sum_{i}\left(\frac{4H}{\Xi}\right)\left|\hat{V}_{\bar{d}_{t,i}}^t-V_{d_i}^{\pi_t}\right|-\frac{\tau_{t}}{2}\|\lambda_t-\lambda_{\tau_{t},\epsilon}^{\star}\|^2\nonumber\\
    &=\frac{(1-\eta_{t}\tau_{t})\|\lambda_{\tau_{t},\epsilon}^{\star}-\lambda_t\|^2-\|\lambda_{\tau_{t},\epsilon}^{\star}-\lambda_{t+1}\|^2}{2\eta_{t}}+\frac{\eta_{t}}{2}D+\sum_{i}\left(\frac{4H}{\Xi}\right)\left|\hat{V}_{\bar{d}_{t,i}}^t-V_{d_i}^{\pi_t}\right|,
\end{align}
where $D\le m(H+\tau_{t}\left(\frac{4H}{\Xi}\right))^2$. Because $\Phi_{t}=\mathrm{KL}_t+\frac{1}{2}\|\lambda_t-\lambda_{\tau_{t},\epsilon}^{\star}\|^2$, combining equation 11 and \eqref{eq:term(1)-2}, it holds
\begin{equation*}
    \Phi_{t+1}\le (1-\eta_{t}\tau_{t})\Phi_t+\frac{\eta_{t}^2}{2}(HC+D)+\underbrace{\eta_{t}\left(\hat{V}_{\bar{y}_t}^t-V_{y_t}^{\pi_t}+\sum_{i}\left(\frac{4H}{\Xi}\right)\left|\hat{V}_{\bar{d}_{t,i}}^t-V_{d_i}^{\pi_t}\right|\right)}_{:=\eta_{t}\delta_t}.
\end{equation*}


Applying the recursion inductively, we get
\begin{align}\label{eq:term-sum}
    \Phi_{t+1}\le& \left(\prod_{j=1}^t \left(1-\eta_{j}\tau_{j}\right)\right)\Phi_1+\sum_{j=1}^t\left[\left(\frac{\eta_j^2}{2}\left(HC+D\right)+\eta_j\delta_j\right)\left(\prod_{k=j+1}^t \left(1-\eta_{k}\tau_{k}\right)\right)\right]\nonumber\\
    \le& \left(\prod_{j=1}^t \left(1-\eta_{j}\tau_{j}\right)\right)\Phi_1+\sum_{j=1}^t\left[\frac{\eta_j^2}{2}\left(HC+D\right)\left(\prod_{k=j+1}^t \left(1-\eta_{k}\tau_{k}\right)\right)\right]\nonumber\\
    &+\sum_{j=1}^t\left[\eta_j\delta_j\left(\prod_{k=j+1}^t \left(1-\eta_{k}\tau_{k}\right)\right)\right]\nonumber\\
    \le& \exp\left(-\sum_{j=1}^t\eta_j\tau_j\right)\Phi_1+\frac{HC+D}{2}\sum_{j=1}^t\left[\eta_j^2\exp\left(-\sum_{k=j+1}^t\eta_{k}\tau_{k}\right)\right]\nonumber\\
    &+\sum_{j=1}^t\eta_j\delta_j\exp\left(-\sum_{k=j+1}^t\eta_{k}\tau_{k}\right).
\end{align}
where the last inequality holds according to $(1-x)\le \exp{(-x)}$ for $x<1$. This completes the result.
\end{proof}

While we have established its theoretical convergence, it doesn't tell us how close our solutions truly are to the optimal policy. Therefore, we prove the error bounds to bridge this gap, linking our theoretical convergence directly to practical performance guarantees.
\begin{customlem}{2}[Per-episode trade-off]\label{lem:error-bounds}
For any $t\ge C''$, any constraint $i$ and any sequence $(\pi_t)_{t\in[T]}$, it holds
\begin{align*}
    \bigl[V_{r}^{\pi^\star}-V_{r}^{\pi_t}\bigr]_{+}&\le
H^{3/2}\,\bigl(2\Phi_t\bigr)^{1/2}+\tau_{t} H\log(A)+\frac{H}{\Xi}\epsilon_{i,t},\\
\max_{i\in[m]}\;\bigl[\alpha_{i}-V_{d_i}^{\pi_t}\bigr]_{+} &\le \bigl[H^{3/2}\,\bigl(2\Phi_t\bigr)^{1/2} +\tau_{t}\left(\frac{4H}{\Xi}\right)-\epsilon_{i,t}\bigr]_{+}.
\end{align*}
\end{customlem}

\begin{proof}We first bound the reward distance between the optimal policy and the actual policy.We present the decomposition as
\begin{equation*}
V_{r}^{\pi^\star}-V_{r}^{\pi_t}=
\underbrace{V_{r}^{\pi^\star}-V_{r}^{\pi_{\tau_{t},\epsilon}^{\star}}}_{(\text{1})}+
\underbrace{V_{r}^{\pi_{\tau_{t},\epsilon}^{\star}}-V_{r}^{\pi_t}}_{(\text{2})}.  
\end{equation*}
We bound terms (1) and (2) respectively. For term (1), to bound the difference between $V_r^{\pi^{\star}}$ and $V_{r}^{\pi_{\tau_{t},\epsilon}^{\star}}$, we first construct a feasible policy for the more constrained problem, i.e., $V_d^\pi(p)\ge \alpha+\epsilon_t$. We define a probabilistic mixed policy for any $t\ge C''$
\begin{equation*}
    \pi^{\text{mix}}=(1-B_t)\pi^\star+B_t\pi^0,
\end{equation*}
where $B_t$ is a Bernoulli distributed random variable with $\frac{\epsilon_{i,t}}{\Xi}$ and $\pi^0$ is the feasible policy under the original optimization problem (as shown in Assumption \ref{assumption}). Then, we have that for any constraint $i$
\begin{align*}
    V_{d_i}^{\pi^{\text{mix}}}&=\left(1-\frac{\epsilon_{i,t}}{\Xi}\right)V_{d_i}^{\pi^\star}+\frac{\epsilon_{i,t}}{\Xi}V_{d_i}^{\pi^0}\\
    &\ge \left(1-\frac{\epsilon_{i,t}}{\Xi}\right)\alpha_i+\frac{\epsilon_{i,t}}{\Xi}d_i^0\\
    &=\alpha_i+\epsilon_{i,t}.
\end{align*}
$\pi^{\text{mix}}$ may not a Markov policy. However, by Lemma \ref{lem:mix policy}, there exists a markov policy $\hat{\pi}^{\text{mix}}$, which has the same performance as $\pi^{\text{mix}}$, which is the feasible policy for the problem below 
\begin{equation}
\max_{\pi\in \Pi}\;V_{r}^\pi+\tau_{t} \mathcal{H}(\pi)
\quad
\text{s.t.}
\quad
V_{d_i}^\pi\ge
\alpha_i+\epsilon_{i,t} \quad(\forall i\in [m]).
\end{equation}
This indicates that
\begin{equation*}
    V_r^{\pi_{\tau_{t},\epsilon}^\star}+\tau_{t} \mathcal{H}(\pi_{\tau_{t},\epsilon}^\star)\ge V_r^{\pi^{\text{mix}}}+\tau_{t} \mathcal{H}(\pi^{\text{mix}}).
\end{equation*}
We then obtain 
\begin{equation*} 
V_r^{\pi_{\tau_{t},\epsilon}^\star}+\tau_{t} \mathcal{H}(\pi_{\tau_{t},\epsilon}^\star)\ge\left(\left(1-\frac{\epsilon_{i,t}}{\Xi}\right)V_r^{\pi^\star}+\frac{\epsilon_{i,t}}{\Xi}V_r^{\pi^0}\right)+\tau_{t} \mathcal{H}(\pi^{\text{mix}})
\end{equation*}
Putting the difference term on the right side, it thus holds that
\begin{align}\label{eq:error-r-1}
V_r^{\pi^{\star}}-V_r^{\pi_{\tau_{t},\epsilon}^\star}&\le \frac{\epsilon_{i,t}}{\Xi}\left(V_r^{\pi^\star}-V_r^{\pi^0}\right)+\tau_{t} \left(\mathcal{H}(\pi_{\tau_{t},\epsilon}^\star)-\mathcal{H}(\pi^{\text{mix}})\right)\nonumber\\
&\le \frac{\epsilon_{i,t}}{\Xi}\cdot H+\tau_{t} H \log(A).
\end{align}

In terms of Term (2), we have 
\begin{align}\label{eq:error-r-2}
(2)&= \sum_{h=1}^{H} \mathbb{E}\left[ \sum_a (\pi_{\tau_{t},h}^\star(a|s) - \pi_{t,h}(a|s)) Q_{r,h}^{\pi_t}(s,a) \mid s_0 \right] \nonumber\\
&= \sum_{h=1}^{H} \sum_{s \in \mathcal{S}} q_{h}^{\pi_{\tau_{t},\epsilon}^{\star}}(s) \sum_{a} (\pi_{\tau_{t},h}^\star(a|s) - \pi_{t,h}(a|s)) Q_{r,h}^{\pi_t}(s,a) \nonumber\\
&\le H \sum_{h=1}^{H} \sum_{s \in \mathcal{S}} q_{h}^{\pi_{\tau_{t},\epsilon}^{\star}}(s) \| \pi_{\tau_{t}}(\cdot|s) - \pi_t(\cdot|s) \|_1 \quad & (\text{since } |Q_{r,h}^{\pi_t}(s,a)| \le H) \nonumber\\
&\le H \sum_{h=1}^{H} \sum_{s \in \mathcal{S}} q_{h}^{\pi_{\tau_{t},\epsilon}^{\star}}(s) \sqrt{2 \mathrm{KL}(\pi_{\tau_{t},h}^\star(\cdot|s),\pi_{t,h}(\cdot|s))} \quad & (\text{by Pinsker's Inequality}) \nonumber\\
&\le \sqrt{2} H  \sqrt{\left( \sum_{s,h} q_{h}^{\pi_{\tau_{t},\epsilon}^{\star}}(s) \right) \left( \sum_{s,h} q_{h}^{\pi_{\tau_{t},\epsilon}^{\star}}(s) \mathrm{KL}(\pi_{\tau_{t},h}^\star(\cdot|s),\pi_{t,h}(\cdot|s))\right)}\nonumber\\
&\le \sqrt{2} H  \sqrt{H \sum_{s,h} q_{h}^{\pi_{\tau_{t},\epsilon}^{\star}}(s) \mathrm{KL}(\pi_{\tau_{t},h}^\star(\cdot|s),\pi_{t,h}(\cdot|s))} \quad & (\text{since } \sum_{s,h} q_{h}^{\pi_{\tau_{t},\epsilon}^{\star}}(s) \le H) \nonumber\\
&= H^{3/2}\sqrt{2\mathrm{KL}_{t}}.
\end{align}
Combining \eqref{eq:error-r-1} and \eqref{eq:error-r-2} together, we obtain
\begin{align*}
 \bigl[V_{r}^{\pi^\star}-V_{r}^{\pi_t}\bigr]_{+}&\le
H^{3/2}\,\bigl(2\mathrm{KL}_{t}\bigr)^{1/2}+\tau_{t} H\log(A)+\frac{H}{\Xi}\epsilon_{i,t}\\
&\le
H^{3/2}\,\bigl(2\Phi_t\bigr)^{1/2}+\tau_{t} H\log(A)+\frac{H}{\Xi}\epsilon_{i,t}.
\end{align*}
Next, we bound the maximum constraint violation between the thresholds and the policy. For any constraint $i\in[m]$, we give the decomposition as
\begin{equation*}
\alpha_i - V_{d_i}^{\pi_t}=
\underbrace{\alpha_i-V_{d_i}^{\pi_{\tau_{t},\epsilon}^{\star}}}_{(3)}+\underbrace{V_{d_i}^{\pi_{\tau_{t},\epsilon}^{\star}}-V_{d_i}^{\pi_t}}_{(4)}.
\end{equation*}
We then bound terms (3) and (4) respectively. The bound for Term (3) is derived from the properties of the saddle point $(\pi_{\tau_{t},\epsilon}^{\star},\lambda_{\tau_{t},\epsilon}^{\star})$. Specially, we use the second inequality from Lemma \ref{lem:SP ineq}, which states that for any $\lambda \in \mathcal{C}$: $ \lambda_{\tau_{t},\epsilon}^{\star\top} (V_{d}^{\pi_{\tau_{t},\epsilon}^{\star}} -\epsilon_t-\alpha) \le \lambda^\top (V_{d}^{\pi_{\tau_{t},\epsilon}^{\star}} -\epsilon_t-\alpha) + \frac{\tau_{t}}{2} (\|\lambda\|^2 - \|\lambda_{\tau_{t},\epsilon}^{\star}\|^2)$. Rearranging this inequality, it holds that: 
\begin{equation}\label{eq:sd-inequal}
    (\lambda-\lambda_{\tau_{t},\epsilon}^{\star})^\top (V_{d}^{\pi_{\tau_{t},\epsilon}^{\star}} -\epsilon_t-\alpha) + \frac{\tau_{t}}{2} (\|\lambda\|^2 - \|\lambda_{\tau_{t},\epsilon}^{\star}\|^2)\ge 0.
\end{equation}
For any constraint $j\in [m]$, we construct a specific $\lambda$ vector by choosing $\lambda_i=\lambda_{\tau_{t},\epsilon,i}^{\star}$ for all $i\neq j$, and $\lambda_j=z$ for some $z \in [0,\left(\frac{4H}{\Xi}\right)]$. Substituting this into \eqref{eq:sd-inequal} reduces it to terms concerning only the $j$-th dimension
\begin{equation*}
    \left(z-\lambda_{\tau_{t},\epsilon,j}^{\star}\right)\left(V_{d_j}^{\pi_{\tau_{t},\epsilon}^{\star}} -\epsilon_t-\alpha_j\right)+\frac{\tau_{t}}{2}\left(z^2-(\lambda_{\tau_{t},\epsilon,j}^{\star})^2\right)\ge 0.
\end{equation*}
The dual solution does not lie on the boundary of the feasible set (i.e. $\lambda_{\tau_{t},\epsilon,j}^{\star}<\frac{4H}{\Xi}$). Thus we can choose a value $z$ such that $\lambda_{\tau_{t},j}^\star<z\le \left(\frac{4H}{\Xi}\right)$, which means 
$z-\lambda_{\tau_{t},j}^\star>0$. After rearranging, we get 
\begin{align}\label{eq:error-d-1}
 (3)=\alpha_j-V_{d_j}^{\pi_{\tau_{t},\epsilon}^{\star}}&\le \frac{\tau_{t}}{2}(z+\lambda_{\tau_{t},\epsilon,j}^{\star})-\epsilon_{j,t}\nonumber\\
 &\le \tau_{t}\left(\frac{4H}{\Xi}\right)-\epsilon_{j,t}.
\end{align}

For term (4), a similar analysis to that of the reward gap (in equation \ref{eq:error-r-2}), we have
\begin{equation}\label{eq:error-d-2}
(4)=V_{d_i}^{\pi_{\tau_{t},\epsilon}^{\star}}-V_{d_i}^{\pi_t}\le H^{3/2}(2\mathrm{KL}_t)^{1/2}. 
\end{equation}
Then we combine the bounds from \eqref{eq:error-d-1} and \eqref{eq:error-d-2} together to get the upper bound below:
\begin{align*}
    \alpha_{i}-V_{d_i}^{\pi_t}&\le
H^{3/2}\,\bigl(2\mathrm{KL}_{t}\bigr)^{1/2}
+\tau_{t}\left(\frac{4H}{\Xi}\right)-\epsilon_{i,t}\\
&\le
H^{3/2}\,\bigl(2\Phi_t\bigr)^{1/2}
+\tau_{t}\left(\frac{4H}{\Xi}\right)-\epsilon_{i,t}.
\end{align*}
Since this holds for any constraint $i$, we can take the positive part on both sides and then the maximum over $i$ to obtain the final result:
\begin{equation*}
    \max_{i\in[m]}\;\bigl[\alpha_{i}-V_{d_i}^{\pi_t}\bigr]_{+} \le \boldsymbol{\bigl[} H^{3/2}\,\bigl(2\Phi_{t}\bigr)^{1/2} +\tau_{t}\left(\frac{4H}{\Xi}\right)-\epsilon_{i,t} \boldsymbol{\bigr]_{+}}.
\end{equation*}
\end{proof}

\subsection{Strong Regret Bounds}\label{app:strong regret bounds}
We are now ready to establish the bounds for strong reward regret and strong constraint violation of Algorithm \ref{alg:main}. 
\begin{customthm}{1}[Bounds for reward regret and constraint violation regret]\label{thm:strong bounds}
    Let $\eta_{t}=t^{-5/6}$, $\tau_{t}=t^{-1/6}$ for $t\ge 1$, and $\epsilon_{i,t} =\frac{18}{5}\sqrt{H^3C_B}\left(t^{-1/6}\cdot \log(SAHt/\delta')^{1/4}\right)$ for all constraint $i$. For a confidence parameter $\delta\in(0,1)$, with probability at least $1-\delta$, when $T$ is sufficiently large, Algorithm \ref{alg:main} achieves the following bounds:
    \begin{equation*}
        \mathcal{R}_{T}(r)\le\tilde{O}(T^{5/6}) \quad and \quad \mathcal{R}_{T}(d)=\tilde{O}(1).
    \end{equation*}
    where $T$ denotes the number of episodes, $C_B=\left(1+\frac{8mH}{\Xi}\right)\left(4H\sqrt{2SA}\left(H\sqrt{S}+H+1\right)\right)+\frac{4mH}{\Xi}\sqrt{2H}$ is a $T$-independent constant and $\tilde{O}$ hides polylogarithmic factors in $(S,A,H,m,\log(T),\log(\frac{1}{\delta}),\Xi)$.
\end{customthm}
\begin{proof}
    For episode $t\ge C''$, according to Lemmas \ref{lem:regularized convergence} and \ref{lem:error-bounds}, we can obtain that for any constraint $i$
    \begin{align*}
        [V_r^{\star}-V_r^{\pi_t}]_{+}&\le H^{3/2}\sqrt{2\Phi_t}+\tau_{t}\log(A)H+\frac{H}{\Xi}\epsilon_{i,t}\nonumber\\
        &\le H^{3/2}\exp{(-\sum_{j=1}^t\eta_{j}\tau_{j} /2)}\sqrt{2\Phi_1}
        +H^{3/2}\sqrt{HC+D}\left(\sum_{j=1}^t\eta_j^2\exp\left(-\sum_{k=j+1}^t \eta_k\tau_k\right)\right)^{1/2}\\
        &\quad+\sqrt{2H^3}\left(\sum_{j=1}^t\eta_j\delta_j\exp\left(-\sum_{k=j+1}^t \eta_k\tau_k\right)\right)^{1/2}
        +\tau_{t}\log(A)H+\frac{H}{\Xi}\epsilon_{i,t}.
    \end{align*}
    Since the strong reward regret $\mathcal{R}_{T}(r)=\sum_{t\in[T]}[V_{r}^{\pi^{\star}}-V_r^{\pi_t}]_+$, it can obtain by summing all the terms over $T$ episodes. Thus, we have
    \begin{align}
        \mathcal{R}_{T}(r)&\le \sum_{t=1}^{C''-1}[V_r^{\pi^\star}-V_r^{\pi_t}]_++\sum_{t=C''}^T \left(H^{3/2}\sqrt{2\Phi_t}+\tau_{t}\log(A)H+\frac{H}{\Xi}\epsilon_{i,t}\right)\nonumber\\
        &\le (C''-1)H\nonumber\\
        &\quad+\sum_{t=C''}^TH^{3/2}\exp{(-\sum_{j=1}^t\eta_{j}\tau_{j} /2)}\sqrt{2\Phi_1}\tag{a}\\
        &\quad+\sum_{t=C''}^TH^{3/2}\sqrt{HC+D}\left(\sum_{j=1}^t\eta_j^2\exp\left(-\sum_{k=j+1}^t \eta_k\tau_k\right)\right)^{1/2}\tag{b}\\
        &\quad+\sum_{t=C''}^T \sqrt{2H^3}\left(\sum_{j=1}^t\eta_j\delta_j\exp\left(-\sum_{k=j+1}^t \eta_k\tau_k\right)\right)^{1/2}\tag{c}\\
        &\quad+\sum_{t=C''}^T  \tau_{t}\log(A)H \tag{d}\\
        &\quad+\sum_{t=C''}^T \frac{H}{\Xi}\epsilon_{i,t}.\tag{e}
    \end{align}

Since $(C''-1)H$ is a constant that is $T$-independent, we will proceed to analysis the bound for term (a)-term (f) separately.
\paragraph{Bounding term (a)} According to the definition of $\Phi_1$, we have $\sqrt{2\Phi_1}\le \sqrt{2H\log(A)+m\left(\frac{4H}{\Xi}\right)^2}:=C'$. Thus, it holds
\begin{align}
(a)&\le H^{3/2}C'\sum_{t=C''}^T\exp(-\sum_{j=1}^t \eta_j\tau_j/2),\nonumber\\
&= H^{3/2}C'\sum_{t=C''}^T\exp(-\sum_{j=1}^t j^{-1}/2).
\end{align}
For the exponent, it yields that
$\sum_{j=1}^t j^{-1}\ge \int_{1}^{t+1} x^{-1}dx = \ln(t+1)$. Substituting this lower bound into the summand yields:
\begin{equation*}
\exp\left(-\frac{1}{2}\sum_{j=1}^t j^{-1}\right)\le \exp\left(-\frac{1}{2}\ln(t+1)\right) = (t+1)^{-1/2}.
\end{equation*}
The sum over $t$ is therefore bounded by
\begin{align*}
\sum_{t=C''}^T (t+1)^{-1/2}
&\le \int_{C''}^{T} (x+1)^{-1/2}dx,\\
&= \left[2(x+1)^{1/2}\right]_{C''}^T,\\
&\le 2\sqrt{2T}.
\end{align*}
Combining this result with the constant, it yields:
\begin{equation*}
(a)\le H^{3/2}C'\left(2\sqrt{2T}\right)=\tilde{O}(\sqrt{T}).
\end{equation*}

    \paragraph{Bounding term (b)} We first calculate the bound of the term $HC+D$ as follows:
    \begin{align*}
    HC+D=&H\exp{\left(\eta_{t} H\left(1+\frac{4mH}{\Xi}+\tau_{t}\log(A)\right)\right)}\left(2A^{\eta_{t}\tau_{t}}H^2\left(1+\frac{4mH}{\Xi}+\tau_{t}\log(A)\right)^2
    + \frac{128\tau_{t}^2\sqrt{A}}{e^2}\right)\\
    &+m\bigl(H + \tau_{t}\,\left(\frac{4H}{\Xi}\right)\bigr)^2\\
    \le&H\exp{\left(H\left(1+\frac{4mH}{\Xi}+\log(A)\right)\right)}\left(2AH^2\left(1+\frac{4mH}{\Xi}+\log(A)\right)^2
    + \frac{128\sqrt{A}}{e^2}\right)\\
    &+m\bigl(H + \frac{4H}{\Xi}\bigr)^2.
    \end{align*}
We can find that the right-hand side of the second inequality is the order of constant, i.e., $\tilde{O}(1)$. Moreover, by Lemma \ref{lem: bound b}, it yields
\begin{equation*}
    (b)\le K\cdot T^{2/3},
\end{equation*}
where $K = \frac{3\sqrt{7}}{2^{2/3}} H^{3/2} \sqrt{HC+D}$.
\paragraph{Bounding term (c)} For term (c), by Lemma \ref{lem: bound_c} and Lemma \ref{lem:bound on cumulative estimation}, we obtain
\begin{equation*}
(c)\le \tilde{O}(T^{5/6}).
\end{equation*}
\paragraph{Bounding term (d)} In terms of (d), it holds
\begin{equation*}
   (d)=\sum_{t=C''}^T\tau_{t} \log(A)H\le \log(A) H \sum_{t=C''}^T t^{-1/6}=\tilde{O}(T^{5/6}). 
\end{equation*}
\paragraph{Bounding term (e)} For term (e), by our setting, 
$$(e)= \frac{H}{\Xi}\sum_{t=C''}^T \epsilon_{i,t}=\tilde{O}(T^{5/6}).$$ 

We now calculate the regret bound for constraint violation. By Lemma \ref{lem:error-bounds}, for episode $t\ge C''$, the per-episode constraint violation is bounded by:
\begin{equation*}
    \max_{i\in[m]}\;\bigl[\alpha_{i}-V_{d_i}^{\pi_t}\bigr]_{+}\le [H^{3/2}\sqrt{2\Phi_t}+\tau_{t}\left(\frac{4H}{\Xi}\right)-\epsilon_{i,t}]_+.
\end{equation*}
Let $P_t:=H^{3/2}\sqrt{2\Phi_t}+\tau_{t}\left(\frac{4H}{\Xi}\right)$. The cumulative constraint violation is bounded by 
\begin{align*}
  R_T(d)&\le \max_{i\in [m]}\sum_{t=1}^{C''-1}[\alpha_i-V_{d_i}^{\pi_t}]_++\sum_{t=C''}^T[P_t-\epsilon_{i,t}]_+\\
  &\le (C''-1)H+\sum_{t=C''}^T[P_t-\epsilon_{i,t}]_+\\
  &\le (C''-1)H+\underbrace{\sum_{t=C''}^T\left[H^{3/2}\sqrt{2\Phi_1}\exp(-\sum_{j=1}^t\eta_{j}\tau_{j}/2)-\epsilon_{i,t}^{(1)}\right]_+}_{(a')}\\
  &\quad+\underbrace{\sum_{t=C''}^T\left[H^{3/2}\sqrt{HC+D}\left(\sum_{j=1}^t \eta_j^2\exp\left(-\sum_{k=j+1}^t \eta_k\tau_k\right)\right)^{1/2}-\epsilon_{i,t}^{(2)}\right]_+}_{(b')}\\
  &\quad+\underbrace{\sum_{t=C''}^T\left[\sqrt{2H^3}\left(\sum_{j=1}^t \eta_j\delta_j\exp{\left(-\sum_{k=j+1}^t\eta_k\tau_k\right)}\right)^{1/2}-\epsilon_{i,t}^{(3)}\right]_+}_{(c')}
  +\underbrace{\sum_{t=C''}^T\left[\left(\frac{4H}{\Xi}\right)\tau_{t}-\epsilon_{i,t}^{(4)}\right]_+}_{(d')},
\end{align*}
where $\epsilon_{i,t}\ge\epsilon_{i,t}^{(1)}+\epsilon_{i,t}^{(2)}+\epsilon_{i,t}^{(3)}+\epsilon_{i,t}^{(4)}$. We will show that with the appropriate choice of $\epsilon_{i,t}^{(i)}$ for any $i\in[4]$, the bound for each term (a')-(d') is $\tilde{O}(1)$. In Section \ref{sec:proof}, $c_1$, $c_2$, $c_3$, $c_4$ denote $H^{3/2}\sqrt{2\Phi_1}$, $H^{3/2}\sqrt{HC+D}$, $\sqrt{2H^3}$ and $\frac{4H}{\Xi}$, respectively.

\paragraph{Bounding term (d')} For term (d'), 
\begin{equation*}
    (d')=\sum_{t=C''}^T\left[\left(\frac{4H}{\Xi}\right)t^{-1/6}-\epsilon_{i,t}^{(4)}\right].
\end{equation*}
With the choice of $\epsilon_{i,t}^{(4)}=\left(\frac{4H}{\Xi}\right)t^{-1/6}$, we immediately get that $(d')=0$.

\paragraph{Bounding term (a')}  For term (a'), 
\begin{equation*}
    (a')=\sum_{t=C''}^T[H^{3/2}\sqrt{2\Phi_1}\exp(-\sum_{j=1}^t j^{-1}/2)-\epsilon_{i,t}^{(1)}]_+.
\end{equation*}
Since $H^{3/2}\sqrt{2\Phi_1}\exp(-\sum_{j=1}^t j^{-1}/2)$ is the same asymptotic order $t^{-1/2}$. We can choose an arbitrarily small constant to multiply by $t^{-1/6}$. By Lemma \ref{lem:limit_comparison_bound}, we obtain $(b')=\tilde{O}(1)$.

\paragraph{Bounding term (b')} For term (b') and any $t\ge C''$, 
\begin{equation*}
    (b')= \sum_{t=C''}^T[H^{3/2}\sqrt{HC+D}\left(\sum_{j=1}^t j^{-5/3}\exp\left(-\sum_{k=j+1}^t k^{-1}\right)\right)^{1/2}-\epsilon_{i,t}^{(2)}]_+.
\end{equation*}
By Lemma \ref{lem: bound b}, it yields that $\left(\sum_{j=1}^t j^{-5/3}\exp\left(-\sum_{k=j+1}^t k^{-1}\right)\right)^{1/2}$ is of the same asymptotic order $t^{-1/3}$. We can choose an arbitrarily small constant to multiply by $t^{-1/6}$. Since the term $t^{-1/3}$ decays strictly faster than $t^{-1/6}$, by Lemma \ref{lem:limit_comparison_bound}, we thus obtain $(b')=\tilde{O}(1)$.

\paragraph{Bounding term (c')} For term (c'), 
\begin{equation*}
    (c')\le\sum_{t=C''}^T\left[\sqrt{2H^3}\left(\sum_{j=1}^t \eta_j\delta_j\exp{\left(-\sum_{k=j+1}^t\eta_k\tau_k\right)}\right)^{1/2}-\epsilon_{i,t}^{(3)}\right]_+.
\end{equation*}
By Lemma \ref{lem: bound_c} and Lemma \ref{lem:limit_comparison_bound}, we pick $\epsilon_{i,t}^{(3)}=\frac{17}{5}\sqrt{H^3C_B}\left(t^{-1/6}\cdot \log(SAHt/\delta')^{1/4}\right)$, where $C_B=\left(1+\frac{8mH}{\Xi}\right)\left(4H\sqrt{2SA}\left(H\sqrt{S}+H+1\right)\right)+\frac{4mH}{\Xi}\sqrt{2H}$. It holds that $\lim_{t\to\infty} \frac{A_t}{\epsilon_{i,t}^{(3)}}< 1$, where $A_t=\sqrt{2H^3}\left(\sum_{j=1}^t \eta_j\delta_j\exp{\left(-\sum_{k=j+1}^t\eta_k\tau_k\right)}\right)^{1/2}$ and thus it holds $(b')=\tilde{O}(1)$.

Sum the terms from $\epsilon_{i,t}^{(1)}$ to $\epsilon_{i,t}^{(4)}$, we find that $\epsilon_{i,t}\ge\epsilon_{i,t}^{(1)}+\epsilon_{i,t}^{(2)}+\epsilon_{i,t}^{(3)}+\epsilon_{i,t}^{(4)}$. Moreover, sum all of the bounds for reward regret and constraint violation together, we get the final result:
$$\mathcal{R}_T(r)=\tilde{O}(T^{5/6})\quad\text{and}\quad\mathcal{R}_T(d)=\tilde{O}(1).$$
\end{proof}

\section{Last-Iterate Convergence}\label{app:last-iterate}
In this section, we present the property of Algorithm \ref{alg:main}, showing that the primal-dual iterates of FlexDOME converge in the last iterate. In contrast to Lemma \ref{lem:regularized convergence} which accounts for estimation errors, we analyze a more fundamental scenario. Here, we assume the model is known, thereby allowing us to neglect the effects of estimation errors. This setting enables a more direct proof of its intrinsic convergence guarantee, as shown in the following lemma.
\begin{customthm}{2}[Convergence for potential functions]\label{lem:last-iterate convergence} Let $\eta_{t}$, $\tau_{t}\le 1$. The policy-dual divergence potential holds
\begin{equation*}
    \Phi_{t+1}\le (1-\eta_{t}\tau_{t})\Phi_t+\frac{\eta_{t}^2}{2}(HC+D)
\end{equation*}
where $C=\exp{\left(\eta_{t} H\left(1+\frac{4mH}{\Xi}+\tau_{t}\log(A)\right)\right)}\left(2A^{\eta_{t}\tau_{t}}H^2\left(1+\frac{4mH}{\Xi}+\tau_{t}\log(A)\right)^2 + \frac{128\tau_{t}^2\sqrt{A}}{e^2}\right) $ and $D=m\bigl(H + \tau_{t}\,\left(\frac{4H}{\Xi}\right)\bigr)^2$.
\end{customthm}
\begin{proof} We recall the definition of the potential function $ \Phi_{t}
\;=\;
\sum_{s,h}\mathbb{P}_{\pi_{\tau_{t},\epsilon}^{\star}}[s_h = s]\,
\mathrm{KL}\Bigl(\pi_{\tau_{t},h}^\star(\cdot\!\mid\!s),\pi_{t,h}(\cdot\!\mid\!s)\Bigr)
+
\frac12\bigl\|\lambda_{\tau_{t},\epsilon}^\star - \lambda_{t}\bigr\|^2$. We first decompose the primal-dual gap,
\begin{equation*}
    \mathcal{L}_{\tau_{t},t}(\pi^\star_{\tau_{t},\epsilon}, \lambda_t) - \mathcal{L}_{\tau_{t},t}(\pi_t, \lambda^\star_{\tau_{t},\epsilon}) = 
    \underbrace{\mathcal{L}_{\tau_{t},t}(\pi^\star_{\tau_{t},\epsilon}, \lambda_t) - \mathcal{L}_{\tau_{t},t}(\pi_t, \lambda_t)}_{\text{(1)}}
    + \underbrace{\mathcal{L}_{\tau_{t},t}(\pi_t, \lambda_t) - \mathcal{L}_{\tau_{t},t}(\pi_t, \lambda^\star_{\tau_{t},\epsilon})}_{\text{(2)}}
\end{equation*}
and we next deal with (1) and (2), separately. 
\paragraph{Bounding term (1)}
\begin{align*}
(\text{1})
&=\mathcal{L}_{\tau_{t},t}\bigl(\pi_{\tau_{t,\epsilon}}^\star,\lambda_t\bigr)-\mathcal{L}_{\tau_{t},t}\bigl(\pi_{t},\lambda_t\bigr)
\\
&=V_{r+\lambda_t^Tg}^{\pi_{\tau_{t,\epsilon}}^\star}-
V_{r+\lambda_t^Tg}^{\pi_{t}}\\
&\quad+\tau_{t}\sum_{s,a,h}q_{h}^{\pi_{t}}(s)\pi_{t,h}(a\!\mid\!s)\log\bigl(\pi_{t,h}(a\!\mid\!s)\bigr)-
\tau_{t}\sum_{s,a,h}q_{h}^{\pi_{\tau_{t,\epsilon}}^\star}(s)\pi_{\tau_{t},\epsilon,h}^\star(a\!\mid\!s)\,\log\bigl(\pi_{\tau_{t},\epsilon,h}^\star(a\!\mid\!s)\bigr)\\
&=V_{r+\lambda_t^Tg+\tau_{t} \psi_t}^{\pi_{\tau_{t,\epsilon}}^\star}-
V_{r+\lambda_t^Tg+\tau_{t}\psi_t}^{\pi_{t}}\\
&\quad+\tau_{t}\sum_{s,a,h}q_{h}^{\pi_{\tau_{t,\epsilon}}^\star}(s)\pi_{\tau_{t},\epsilon,h}^\star(a\!\mid\!s)\log\bigl(\pi_{t,h}(a\!\mid\!s)\bigr)-
\tau_{t}\sum_{s,a,h}q_{h}^{\pi_{\tau_{t,\epsilon}}^\star}(s)\pi_{\tau_{t},\epsilon,h}^\star(a\!\mid\!s)\,\log\bigl(\pi_{\tau_{t},\epsilon,h}^\star(a\!\mid\!s)\bigr)\\
&=V_{r+\lambda_t^Tg+\tau_{t} \psi_t}^{\pi_{\tau_{t,\epsilon}}^\star}-
V_{r+\lambda_t^Tg+\tau_{t}\psi_t}^{\pi_{t}}-\tau_{t} \textup{KL}_t\\
&=V_{y_t}^{\pi_{\tau_{t,\epsilon}}^\star}-V_{y_t}^{\pi_t}-\tau_{t} \textup{KL}_t.
\end{align*}

By Lemma \ref{lem:last-iterate convergence-0}, we have
\begin{align*}
V_{y_t}^{\pi_{\tau_{t,\epsilon}}^\star}-V_{y_t}^{\pi_t}&=\sum_{s,h}q_h^{\pi_{\tau_{t},\epsilon}^{\star}}\bigl\langle Q_{y_t,h}^{\pi_{t}}(s,\cdot),\pi_{\tau_{t},h}^\star(\cdot\mid s)-\pi_{t,h}(\cdot\mid s)
\bigr\rangle\\
&\le\sum_{s,h}q_h^{\pi_{\tau_{t},\epsilon}^{\star}}\Bigl(\frac{\mathrm{KL}_{t,h}(s)-\mathrm{KL}_{t+1,h}(s)}{\eta_{t}}+
\frac{\eta_{t}}{2}
\sum_{a}\pi_{t,h}(a\!\mid\!s)\exp\bigl(Q_{y_t,h}^{\pi_t}(s,a)\bigr)\,
Q_{y_t,h}^{\pi_t}(s,a)^{2}\Bigl).
\end{align*}

According to Lemma \ref{lem:q-value-bounds}, it holds that $\sum_{a}\pi_{t,h}(a\!\mid\!s)\exp\bigl(Q_{y_t,h}^{\pi_t}(s,a)\bigr)\,
Q_{y_t,h}^{\pi_t}(s,a)^{2}\le C$, where $C=\exp{\left(\eta_{t} H\left(1+\frac{4mH}{\Xi}+\tau_{t}\log(A)\right)\right)}\left(2A^{\eta_{t}\tau_{t}}H^2\left(1+\frac{4mH}{\Xi}+\tau_{t}\log(A)\right)^2 + \frac{128\tau_{t}^2\sqrt{A}}{e^2}\right)$. Then we obtain 
\begin{align*}
V_{y_t}^{\pi_{\tau_{t,\epsilon}}^\star}-V_{y_t}^{\pi_t}
&\le\sum_{s,h}q_h^{\pi_{\tau_{t},\epsilon}^{\star}}\Bigl(\frac{\mathrm{KL}_{t,h}(s)-\mathrm{KL}_{t+1,h}(s)}{\eta_{t}}+
\frac{\eta_{t}}{2}C\Bigl)\\
&=
\frac{\mathrm{KL}_{t}-\mathrm{KL}_{t+1}}{\eta_{t}}+
\frac{\eta_{t} H}{2}C.
\end{align*}

Therefore, we get 
\begin{equation}\label{eq:convergence-1}
(\text{1})=V_{y_t}^{\pi_{\tau_{t,\epsilon}}^\star}-
V_{y_t}^{\pi_t}-\tau_{t}\mathrm{KL}_t\le\frac{\mathrm{KL}_t - \mathrm{KL}_{t+1}}{\eta_{t}}+\frac{\eta_{t} H}{2}C-\tau_{t}\mathrm{KL}_t=
\frac{1-\eta_{t}\tau_{t}}{\eta_{t}}\mathrm{KL}_t-\frac{\mathrm{KL}_{t+1}}{\eta_{t}}+\frac{\eta_{t} H}{2}C.
\end{equation}

\paragraph{Bounding term (2)}
\begin{align*}
(\text{2})&=\mathcal{L}_{\tau_{t},t}(\pi_{t},\lambda_{t})-\mathcal{L}_{\tau_{t},t}(\pi_{t},\lambda_{\tau_{t}, \epsilon}^\star)\\ 
&=\sum_{i\in[m]}\lambda_{t,i}(V_{d_i}^{\pi_t}-\epsilon_t-\alpha_i)-\sum_{i\in[m]}\lambda_{\tau_{t},\epsilon,i}^\star(V_{d_i}^{\pi_t}-\epsilon_t-\alpha_i)+\frac{\tau_{t}}{2}\|\lambda_t\|^2-\frac{\tau_{t}}{2}\|\lambda_{\tau_{t},\epsilon}^\star\|^2\\
&=\sum_{i\in[m]}(\lambda_{t,i}-\lambda_{\tau_{t},\epsilon,i}^\star)(V_{d_i}^{\pi_t}-\epsilon_t-\alpha_i+\tau_{t} \lambda_{t,i})-\frac{\tau_{t}}{2}\|\lambda_t-\lambda_{\tau_{t},\epsilon}^\star\|^2\\
&\le \frac{\|\lambda_{\tau_{t}, \epsilon}^\star - \lambda_{t}\|^2 - \|\lambda_{\tau_{t}, \epsilon}^\star - \lambda_{t+1}\|^2}{2\eta_{t}}+\frac{\eta_{t}}{2}\|V_{d}^{\pi_t}-\epsilon_t-\alpha+\tau_{t} \lambda_{t}\|^2-\frac{\tau_{t}}{2}\|\lambda_t-\lambda_{\tau_{t},\epsilon}^\star\|^2.
\end{align*}
Since $\|V_{d}^{\pi_t}-\epsilon_t-\alpha\|\le \sqrt{m}H$ and $\|\lambda_t\|\le \sqrt{m}\left(\frac{4H}{\Xi}\right)$, we have $\|V_{d_i}^{\pi_t}-\epsilon_t-\alpha_i+\tau_{t} \lambda_{t,i}\|^2\le D$, where $D=m\bigl(H + \tau_{t}\,\left(\frac{4H}{\Xi}\right)\bigr)^2$. Hence, it holds that
\begin{align}\label{eq:convergence-2}
(\text{2})&\le
\frac{\|\lambda_{\tau_{t}, \epsilon}^\star-\lambda_{t}\|^2 - \|\lambda_{\tau_{t}, \epsilon}^\star - \lambda_{t+1}\|^2}{2\eta_{t}}+
\frac{\eta_{t}}{2}D
-\frac{\tau_{t}}{2}\|\lambda_t-\lambda_{\tau_{t},\epsilon}^\star\|^2\nonumber\\
&=\frac{1-\eta_{t}\tau_{t}}{2\eta_{t}}\|\lambda_{\tau_{t}, \epsilon}^\star - \lambda_{t}\|^2-\frac{1}{2\eta_{t}}\|\lambda_{\tau_{t}, \epsilon}^\star - \lambda_{t+1}\|^2+\frac{\eta_{t}}{2}D.
\end{align}

By combining \eqref{eq:convergence-1} and \eqref{eq:convergence-2}, we obtain
\begin{align*}
    \Phi_{t+1}&=\mathrm{KL}_{t+1}+\frac{1}{2}\|\lambda_{t+1}-\lambda_{\tau_{t},\epsilon}^\star\|^2\\
    &\le (1-\eta_{t}\tau_{t})(\mathrm{KL}_t+\frac{1}{2}\|\lambda_{t}-\lambda_{\tau_{t},\epsilon}^\star\|^2)+\frac{\eta_{t}^2}{2}(HC+D)-\eta_{t}((1)+(2))\\
    &\le (1-\eta_{t}\tau_{t})\Phi_t+\frac{\eta_{t}^2}{2}(HC+D),\quad & (\text{since }(1)+(2)\ge 0) \nonumber
\end{align*}
where $HC+D=H\exp{\left(\eta_{t} H\left(1+\frac{4mH}{\Xi}+\tau_{t}\log(A)\right)\right)}\left(2A^{\eta_{t}\tau_{t}}H^2\left(1+\frac{4mH}{\Xi}+\tau_{t}\log(A)\right)^2 + \frac{128\tau_{t}^2\sqrt{A}}{e^2}\right)+m\bigl(H + \tau_{t}\,\left(\frac{4H}{\Xi}\right)\bigr)^2$.
\end{proof}

Building upon the linear convergence of the regularized scheme established in Lemma \ref{lem:last-iterate convergence} and error bounds guaranteed by Lemma \ref{lem:error-bounds}, we now demonstrate that the iterates converge to the optimal policy of the original and unregularised problem. We first prove the following lemma and then provide the last-iterate convergence guarantee.

\begin{lemma}[Asymptotic bound on the recursive error sum]\label{lem: bound for b}
Let $\varepsilon \in (0, 1)$ be a sufficiently small positive number. Assume the step-size and regularization parameters satisfy $\eta_j = \Theta(\varepsilon^3)$ and $\tau_j = \Theta(\varepsilon)$ for all $j \in [t]$. Further, assume the number of steps $t = \Omega(\varepsilon^{-4}\log(1/\varepsilon))$. Then the following bound holds:
\begin{equation*}
    \sum_{j=1}^{t} \eta_j^2 \exp\left(-\sum_{k=j+1}^{t} \eta_k\tau_k \right) = \Theta(\varepsilon^2).
\end{equation*}
\end{lemma}

\begin{proof}
Let the expression be denoted by $S_t$. By the definitions of asymptotic notation, there exist positive constants $c_{\eta,1}, c_{\eta,2}, c_{\tau,1}, c_{\tau,2}$ such that for all $j \in [t]$:
\begin{align*}
    c_{\eta,1}\varepsilon^3 \le \eta_j \le c_{\eta,2}\varepsilon^3 \quad \text{and} \quad c_{\tau,1}\varepsilon \le \tau_j \le c_{\tau,2}\varepsilon.
\end{align*}
Let $C_1 := c_{\eta,1}c_{\tau,1}$. The sum in the exponent is bounded below by $\sum_{k=j+1}^{t} \eta_k\tau_k \ge (t-j)C_1\varepsilon^4$. We then have:
\begin{align*}
    S_t &\le \sum_{j=1}^{t} (c_{\eta,2}^2 \varepsilon^6) \exp\left(-(t-j)C_1\varepsilon^4\right) \\
    &= c_{\eta,2}^2 \varepsilon^6 \sum_{m=0}^{t-1} \left(e^{-C_1\varepsilon^4}\right)^m \tag{$m = t-j$} \\
    &\le c_{\eta,2}^2 \varepsilon^6 \left( \frac{1}{1 - e^{-C_1\varepsilon^4}} \right).
\end{align*}
Since $1-e^{-x} = \Theta(x)$ for small $x>0$, the term in the parenthesis is $O(\varepsilon^{-4})$. Thus, $S_t \le c_{\eta,2}^2 \varepsilon^6 \cdot O(\varepsilon^{-4}) = O(\varepsilon^2)$. 

Let $C_2 := c_{\eta,2}c_{\tau,2}$. The sum in the exponent is bounded above by $\sum_{k=j+1}^{t} \eta_k\tau_k \le mC_2\varepsilon^4$. We then have:
\begin{align*}
    S_t &\ge \sum_{j=1}^{t} (c_{\eta,1}^2 \varepsilon^6) \exp\left(-mC_2\varepsilon^4\right) \\
    &= c_{\eta,1}^2 \varepsilon^6 \sum_{m=0}^{t-1} \left(e^{-C_2\varepsilon^4}\right)^m.
\end{align*}
The finite geometric sum is $\frac{1 - r^t}{1-r}$, where $r=e^{-C_2\varepsilon^4}$. As $t=\Omega(\varepsilon^{-4}\log(1/\varepsilon))$, the term $tC_2\varepsilon^4 = \Omega(\log(1/\varepsilon))$ grows sufficiently large as $\varepsilon \to 0$. Therefore, $r^t = \exp(-tC_2\varepsilon^4)$ approaches 0, which implies $1-r^t$ is bounded below by a positive constant for sufficiently small $\varepsilon$. The denominator $1-r$ is $\Theta(\varepsilon^4)$. Thus, the sum is $\Omega(\varepsilon^{-4})$.
Combining these terms, we obtain the lower bound: $S_t \ge c_{\eta,1}^2 \varepsilon^6 \cdot \Omega(\varepsilon^{-4}) = \Omega(\varepsilon^2)$. This completes the proof.
\end{proof}

Building upon the lemmas above, we now give the guarantee of last-iterate convergence.

\begin{theorem}[Last-iterate convergence] Conditioned on Assumption \ref{assumption}, for small $\varepsilon>0$ and $t=\Omega(\varepsilon^{-4}\log(1/\varepsilon))$, if $\eta_{t}=\Theta(\varepsilon^{3})$, $\tau_{t}=\Theta(\varepsilon)$ and $\epsilon_{i,t}=\Theta(\varepsilon)$ for all constraint $i$, then we have
\begin{equation*}
\bigl[V_{r}^{\pi^\star}-V_{r}^{\pi_t}\bigr]_{+}\le \Theta(\varepsilon),
\quad
\bigl[\alpha_{i}-V_{d_i}^{\pi_t}\bigr]_{+}=0
\quad
(\forall\,i\in[m]).
\end{equation*}
\end{theorem}
\begin{proof} According to Lemma \ref{lem:last-iterate convergence}, we have $\Phi_{t+1}\le (1-\eta_{t}\tau_{t})\Phi_t+\frac{\eta_{t}^2}{2}(HC+D)$, and thus it holds that
\begin{align*}
\Phi_{t+1}&\le \left(\prod_{j=1}^t \left(1-\eta_{j}\tau_{j}\right)\right)\Phi_1+\sum_{j=1}^t\left[\frac{\eta_j^2}{2}\left(HC+D\right)\left(\prod_{k=j+1}^t \left(1-\eta_{k}\tau_{k}\right)\right)\right]\\
&\le \left(\prod_{j=1}^t \left(1-\eta_{j}\tau_{j}\right)\right)\Phi_1+\frac{HC+D}{2}\sum_{j=1}^t\left[\eta_j^2\left(\prod_{k=j+1}^t \left(1-\eta_{k}\tau_{k}\right)\right)\right]\\
&\le \exp\left(-\sum_{j=1}^t\eta_j\tau_j\right)\Phi_1+\frac{HC+D}{2}\sum_{j=1}^t\left[\eta_j^2\exp\left(-\sum_{k=j+1}^t\eta_{k}\tau_{k}\right)\right].
\end{align*}
Combining Lemmas \ref{lem:convergence} and \ref{lem:error-bounds}, we obtain the reward distance between the optimal policy and the exact policy as
\begin{align}
\bigl[V_{r}^{\pi^\star}-V_{r}^{\pi_t}\bigr]_{+}&\lesssim H^{3/2}\exp{(-\sum_{j=1}^t\eta_{j}\tau_{j} /2)}\sqrt{\Phi_1}\tag{a}\\
&\quad+H^{3/2}\sqrt{HC+D}\left(\sum_{j=1}^t\eta_j^2\exp\left(-\sum_{k=j+1}^t \eta_k\tau_k\right)\right)^{1/2}\tag{b}\\
&\quad+\tau_{t} H\log(A)\tag{c}\\
&\quad+\frac{H}{\Xi}\epsilon_t.\tag{d}
\end{align}
We now analyze the order of each term with the chosen parameters $\tau_{t}=\Theta(\varepsilon)$, $\eta_{t}=\Theta(\varepsilon^{3})$ and $t=\Omega(\varepsilon^{-4}\log(1/\varepsilon))$. We discuss each term individually.

For term (a), the product in the exponent scales as:
\begin{equation*}
\sum_{j=1}^t\eta_j\tau_j = \Theta(\varepsilon^{3}) \cdot \Theta(\varepsilon) \cdot \Omega(\varepsilon^{-4}\log(1/\varepsilon)) = \Omega(\log(1/\varepsilon)).
\end{equation*}
As $\varepsilon$ goes to zero, the exponent $\sum_{j=1}^t\eta_j\tau_j$ grows sufficiently large, causing $\exp{(-\sum_{j=1}^t\eta_j\tau_j/2)}$ to decay to order $O(\varepsilon)$.
For potential term $\Phi_1$, we have $\Phi_1\le (H\log(A)+\frac{1}{2}m\left(\frac{4H}{\Xi}\right)^2)^{1/2}$. Thus, it shows
\begin{equation}
(a)= O(\varepsilon).
\end{equation}

For term (b), by Lemma \ref{lem: bound for b}, the term scales as $\sqrt{\eta_t/\tau_t}$. We have
\begin{equation*}
\left(\sum_{j=1}^t\eta_j^2\exp\left(-\sum_{k=j+1}^t \eta_k\tau_k\right)\right)^{1/2}=\Theta\left(\sqrt{\frac{\varepsilon^3}{\varepsilon}}\right)=\Theta(\varepsilon).
\end{equation*}
Note that $H^{3/2}\sqrt{HC+D}$ is a constant. Thus, we have
\begin{equation}
(b)= \Theta(\varepsilon).
\end{equation}

In terms of (c), with the new parameter choice, we get
\begin{equation}
(c)=\tau_{t} H\log(A)=\Theta(\varepsilon).
\end{equation}

For term (d), it holds
\begin{equation}
(d)=\frac{H}{\Xi}\epsilon_{i,t}=\Theta(\varepsilon).
\end{equation}

By Lemma \ref{lem:convergence} and Lemma \ref{lem:error-bounds}, the constraint violation between the thresholds and the policy satisfies
\begin{align}
\bigl[\alpha_{i}-V_{d_i}^{\pi_t}\bigr]_{+}&\lesssim \big[H^{3/2}\exp{(-\sum_{j=1}^t\eta_{j}\tau_{j} /2)}\sqrt{\Phi_1}\tag{a'}\\
&\quad+H^{3/2}\sqrt{HC+D}\left(\sum_{j=1}^t\eta_j^2\exp\left(-\sum_{k=j+1}^t \eta_k\tau_k\right)\right)^{1/2}\tag{b'}\\
&\quad+\left(\frac{4H}{\Xi}\right)\tau_{t}\tag{c'}\\
&\quad-\epsilon_{i,t}\big]_+.\tag{d'}
\end{align}

The bounds of terms (a') and (b') are equal to those in the case of reward analysis. For term (c'), with the choice of regularized parameter $\tau_{t}$, we have
\begin{equation}
(c')=\tau_{t}\left(\frac{4H}{\Xi}\right)=\Theta(\varepsilon).
\end{equation}

Let $E_t = (a') + (b') + (c')$. Here, $(b')$ and $(c')$ are dominant terms of order $\Theta(\varepsilon)$, and $(a')$ is upper bounded by $O(\varepsilon)$. Specifically, let $(b') \le C_{b}\varepsilon$, $(c') \le C_{c}\varepsilon$, and $(a') \le C_{a}\varepsilon$. The total error is bounded by $E_t \le (C_{a} + C_{b} + C_{c})\varepsilon$.

The safety margin is chosen as $\epsilon_{i,t} = C_\epsilon \varepsilon$ for a constant $C_\epsilon > 0$. We select $C_\epsilon$ strictly larger than the sum of these constants, i.e., $C_\epsilon > C_{a} + C_{b} + C_{c}$. The expression is therefore bounded as:\begin{equation*}E_t - \epsilon_{i,t} \le (C_{a} + C_{b} + C_{c} - C_\epsilon)\varepsilon < 0.\end{equation*}

Consequently, for all sufficiently small $\varepsilon$, the term $E_t - \epsilon_{i,t}$ is strictly negative. This leads to the following:
\begin{equation*}
\bigl[\alpha_{i}-V_{d_i}^{\pi_t}\bigr]_{+} \lesssim [ E_t - \epsilon_{i,t} ]_{+} = 0.
\end{equation*}
This means $\bigl[\alpha_{i}-V_{d_i}^{\pi_t}\bigr]_{+}= 0$. This completes the proof.
\end{proof}

\section{Additional Details of Experiments}\label{app:experiment}

All experiments are conducted in a randomly generated CMDP, with results averaged over five distinct random seeds. The environment is defined by a state space of $S=20$ states, an action space of $A=5$ actions, a finite horizon of $H=5$ steps and $m=1$ constraint. The environment's dynamics are stochastic; for each state-action pair $(s,a)$ and step $h$, the transition probabilities $\tilde{p}_h(\cdot|s, a)$ are sampled from a Dirichlet distribution with a low concentration parameter of 0.1 to foster sparse transitions.

The learning challenge is shaped by the conflicting design of the reward and constraint functions. 
The reward $\tilde{r}_h(s,a)$ is binary. At initialization, we independently draw $\tilde r_h(s,a)\sim\mathrm{Bernoulli}(0.5)$ for each step $h$ and state-action pair $(s,a)$. Hence, each $\tilde{r}_h(s,a)\in\{0,1\}$. The constraint value $\tilde{d}_h(s,a)$ is defined in opposition to the reward function: $\tilde{d}_h(s,a)=1-\tilde{r}_h(s,a)$. This design creates a challenging learning problem where the agent must balance the conflicting objectives of maximizing rewards while satisfying the constraint~\citep{moskovitz2023reload}. At the beginning of each run, the initial state $s_0$ is selected uniformly at random and remains fixed for all subsequent episodes.

We assess algorithm performance under two threshold scenarios. In the fixed-threshold setting, the threshold $\alpha$ is set to half of the maximum achievable expected constraint value: $\alpha=\frac{1}{2}\max_{\pi\in \Pi} V_{d}^\pi$. This ensures the constraint is both feasible and non-trivial. To model more dynamic conditions, the stochastic-threshold setting samples a constraint value $\alpha_t$ for each episode $t$ from a Normal distribution, $\alpha_t \sim \mathcal{N}(\alpha,(0.5\alpha)^2)$, centered at the fixed threshold value, with a standard deviation equal to half of its mean.

Each algorithm is executed for $T=80000$ episodes, with the confidence parameter $\delta=0.1$. To effectively translate theoretical guarantees into practical performance, we introduce empirically tuned scaling factors. The exploration bonus is scaled by a factor of $10^{-3}$, akin to that of~\cite{kitamura2024policy}. Similarly, the safety margin is scaled by a factor of $10^{-5}$ to mitigate the over-conservatism of the theoretical bound and observe the algorithms' behavior in relatively smaller episodes. To validate these choices, we conducted a sensitivity analysis on the exploration bonus ($c_b$) and safety margin ($c_\epsilon$) scalers, as shown in Figure \ref{fig:sensitivity}. The results reveal a trade-off: increasing  $c_\epsilon$ effectively suppresses strong constraint violations but induces over-conservatism that inflates reward regret. Our selected margin strikes an optimal balance, clamping violations near zero without the significant regret penalties associated with looser theoretical bounds. Similarly, the exploration bonus scaler calibrates the magnitude of optimism with empirical uncertainty.

\begin{figure}[htp]
    \centering
    \includegraphics[width=0.88\linewidth]{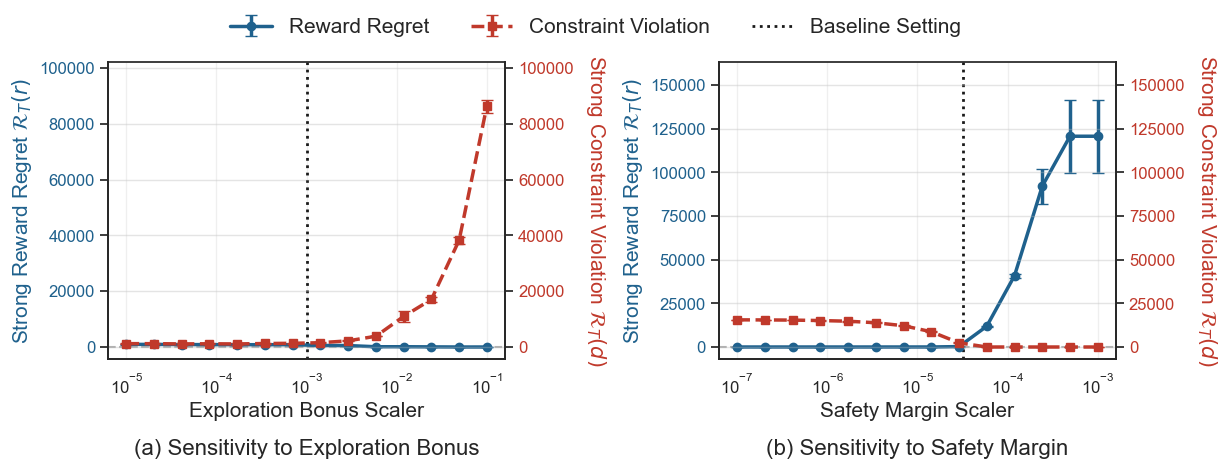}
    \caption{Impact of (a) exploration bonus scaler and (b) safety margin scaler on strong regret and violation. Vertical lines denote the selected baseline parameters. Results are averaged over 5 random seeds with standard error bands.}
    \label{fig:sensitivity}
\end{figure}

All experiments were performed on a Lenovo ThinkBook 14 G5+ APO with an AMD Ryzen 7 7840H.

\section{Declaration on Large Language Models}

Large Language Models were used for (1) polishing the wording of the manuscript for clarity and readability, (2) brainstorming about algorithm names and their abbreviations, and (3) assisting in formalizing proof sketches into some lemma statements, which is later manually checked to ensure correctness.

\end{document}